\documentclass[letterpaper, 10 pt, conference]{formatting/ieeeconf}
\IEEEoverridecommandlockouts                              
\let\labelindent\relax
\usepackage{amsfonts}       

\usepackage{amsthm}
\usepackage{mathtools}      
\usepackage{amssymb}        
\usepackage{filecontents}
\usepackage{graphicx}       
\usepackage{marginnote}     
\usepackage{marvosym}       
\usepackage{overpic}        
\usepackage{tabularx}
\usepackage{cite}
\usepackage{color}
\usepackage[linesnumbered,algoruled,boxed,lined,noend]{algorithm2e}
\usepackage[font=small,labelfont=bf]{caption}

\usepackage{float} 
\usepackage{wrapfig}
\usepackage{comment}

\usepackage{tikz} 
\usepackage[a-2b,mathxmp]{pdfx}[2018/12/22]

\usepackage{caption}
\usepackage{subcaption}

\usepackage[normalem]{ulem}
\usepackage{epstopdf}
\usepackage{enumitem}
\usepackage[normalem]{ulem}
\usepackage{lipsum}
\usepackage[a-2b,mathxmp]{pdfx}[2018/12/22]

\newcommand{\set}[1]{\{#1\}}

\def\ours{\textsc{RVG}\xspace}
\def\layerrange{[\theta_{lb}, \theta_{ub}]}
\def\VG{\mathcal{VG}\xspace}
\def\Layer{\mathcal{L}\textit{ayer}\xspace}
\def\RVG{\mathcal{RVG}\xspace}

\usepackage{xcolor}
\usepackage{xargs} 
\newif\ifdraft
\draftfalse

\ifdraft
\usepackage[paperheight=11in,paperwidth=9.5in,
			left=1.25in,right=1.25in,
			top=0.75in,bottom=0.75in,
			heightrounded,marginparwidth=1.2in,
			marginparsep=0.05in]{geometry}
\usepackage{xcolor}
\usepackage{xargs} 
\usepackage[textsize=footnotesize]{todonotes}
\newcommandx{\nt}[2][1=]{\todo[linecolor=red,
			backgroundcolor=red!10,bordercolor=red,#1]{#2}}
\newcommandx{\jy}[2][1=]{\todo[linecolor=green,
			backgroundcolor=green!10,bordercolor=green,#1]{JY:#2}}
\newcommandx{\dz}[2][1=]{\todo[linecolor=red,
			backgroundcolor=red!10,bordercolor=red,#1]{DZ:#2}}
\else
\newcommand{\nt}[1]{{}}
\newcommand{\jy}[1]{{}}
\newcommand{\dz}[1]{{}}
\fi

\newif\iftwocolumn
\twocolumntrue

\newtheorem{theorem}{Theorem}[section]
\theoremstyle{definition}

\theoremstyle{remark}

\SetKwProg{Fn}{Function}{}{}
\SetKwComment{Comment}{$\triangleright$\ }{}

\makeatletter
\def\subsubsection{\@startsection{subsubsection}
                                 {3}
                                 {\z@ \hspace*{1mm}}
                                 {0ex plus 0.1ex minus 0.1ex}
                                 {0ex}
                                 {\normalfont\normalsize\itshape}}
\makeatother

\title{
Asymptotically-Optimal Multi-Query Path Planning for a Polygonal Robot
}
\author{
Duo Zhang \qquad Zihe Ye \qquad Jingjin Yu
\thanks{D. Zhang, Z. Ye, and J. Yu are with the Department of 
Computer Science, Rutgers, the State University of New Jersey, Piscataway, NJ, USA. 
Emails: {\tt\small \{duo.zhang, zihe.ye, jingjin.yu\}@rutgers.edu}.
}
\thanks{This work was supported in part by NSF awards IIS-1845888, IIS-2132972, and CCF-2309866. 
}
}

\begin{document}

\maketitle
\thispagestyle{empty}
\pagestyle{empty}

\pdfcompresslevel=9
\pdfobjcompresslevel=2

\begin{abstract}
Shortest-path roadmaps, also known as reduced visibility graphs, provide a highly efficient multi-query method for computing optimal paths in two-dimensional environments. Combined with Minkowski sum computations, shortest-path roadmaps can compute optimal paths for a translating robot in 2D. In this study, we explore the intuitive idea of stacking up a set of reduced visibility graphs at different orientations for a polygonal holonomic robot to support the fast computation of near-optimal paths, allowing simultaneous 2D translation and rotation. The resulting algorithm, \emph{rotation-stacked visibility graph} (\ours), is shown to be resolution-complete and asymptotically optimal. Extensive computational experiments show \ours significantly outperforms state-of-the-art single- and multi-query sampling-based methods on both computation time and solution optimality fronts. 
Source code and supplementary materials are available at \href{https://github.com/arc-l/rvg}{\texttt{\textcolor{blue}{https://github.com/arc-l/rvg}}}. 
\end{abstract}

\section{Introduction}\label{sec:intro}
In many robotics applications, it is desirable to compute high-quality collision-free paths quickly for a translating and rotating object in two dimensions. A prominent example is the path planning for robotic vehicles, which are often rectangular (including squares), to navigate in indoor (e.g., warehouses) or outdoor environments.
Another frequently encountered challenge is finding such paths for manipulating objects on flat surfaces. This can happen when heavy furniture is being moved around in a house or when some objects to be rearranged on a tabletop cannot be lifted above other objects and directly transported around (e.g., \cite{huang2023toward}, see Fig.~\ref{fig:app}). In these cases, a short (in terms of Euclidean distance) collision-free path is often preferred. 
Whereas such problems can be solved using combinatorial approaches \cite{avnaim1988practical, alt1990approximate, halperin1996near, agarwal1999motion}, sampling-based algorithms \cite{kuffner2000rrt, kavraki1996probabilistic,gammell2020batch,strub2020adaptively, bohlin2000path,salzman2013motion,salzman2016asymptotically, karaman2011sampling, Kavraki2008} and gradient-based methods \cite{schulman2013finding, biegler2009large, hauser2021semi, zhang2023provably, liang2024second, pan2024provably}, either solution quality could be improved or computation time needs to be accelerated. In addition, the time required to reach a solution can vary greatly for random sampling-based methods. This raises a natural question: Can a specialized high-performance algorithm be developed to plan high-quality paths for moving a polygonal robot in 2D, amongst polygonal obstacles? 

\emph{Shortest-path roadmaps} \cite{nilsson1969mobile}, also called \emph{reduced visibility graphs} \cite{latombe2012robot}, find the shortest path connecting two points in two dimensions where bitangents can be computed for the obstacles (including the environment boundary). Leveraging shortest-path roadmaps, it is shown \cite{lozano1979algorithm} that 2D shortest paths can be effectively computed for translating a polygon among static polygonal obstacles. This is achieved by first explicitly computing the \emph{configuration space} or $\mathcal C$-space \cite{lozano1990spatial} of the movable polygon by taking the Minkowski sum \cite{hadwiger1950minkowskische} between the polygon and the environment. Then, a shortest-path roadmap can be computed over the resulting free configuration space $\mathcal C_{free}$, allowing arbitrary shortest-path queries between two points within $\mathcal C_{free}$ to be quickly resolved. 
In the same study \cite{lozano1979algorithm}, 2D rotations are also considered to a limited extent through combining the start and goal rotations, which is sub-optimal and incomplete. It is hypothesized, but not explored further, that dividing the rotational degree of freedom can achieve a better outcome. 
\begin{figure}[t]
\centering
\includegraphics[width=\linewidth]{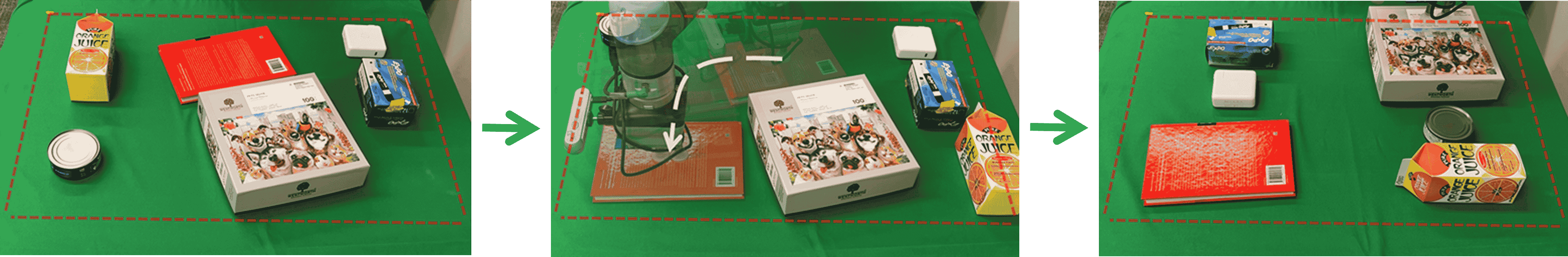}
\caption{A motivating example in developing \ours. In object rearrangement, high-quality paths often need to be planned to translate and rotate objects in a 2D workspace, e.g., a tabletop. Images are reproduced from \cite{huang2023toward}.}
\label{fig:app}
\end{figure}

Motivated by vast applications and inspired by \cite{lozano1979algorithm}, in this work, we develop \emph{rotation-stacked reduced visibility graph} ($\RVG$) for computing asymptotically-optimal, collision-free paths for moving a polygon (equivalently, a holonomic robot with a polygonal footprint) in two dimensions. Roughly speaking, \ours (our method for computing $\RVG$) slices the rotational degree of the $SE(2)$ configuration space. After computing the Minkowski sum of a rotated polygon with the environment for each rotation slice, the corresponding reduced visibility graph is constructed. Then, \ours carefully connects these reduced visibility graphs to yield a single rotation-stacked reduced visibility graph. We prove that \ours is resolution-complete and its solutions converge to the optimal solution in the limit. Perhaps more importantly, we show that \ours, as a multi-query method, delivers better computation time and solution optimality than state-of-the-art single- and multi-query sampling-based methods.

\section{Preliminaries}\label{sec:problem}
\subsection{Problem Definition}\label{subsec:prob}
Let $\mathcal W \subset \mathbb R^2$ be a compact (i.e., closed and bounded) workspace. There is a set of polygonal obstacles $\{O_i\}$. For all $i$,  $O_i \subset \mathcal W$ is compact. For $i \ne j$, $O_i \cap O_j = \varnothing$. Let $O = \bigcup_{O_i\in \{O_i\}} O_i$. A polygonal holonomic robot $r$ resides in the $\mathcal W \backslash O$, assuming a configuration $q = (x, y, \theta) \in SE(2)$. Let the free configuration space for the robot be $\mathcal C_{free}$ and let its closure be $\overline{\mathcal C_{free}}$.  
A \emph{feasible path} for the robot is a continuous map $\tau: [0, 1] \rightarrow \overline{\mathcal C_{free}}$. Intuitively, the robot may touch obstacles but may not penetrate any obstacles. 

We make a general position assumption that the robot may not be ``sandwiched'' between two obstacles in $SE(2)$, simultaneously touching both. Formally, for any feasible path $\tau$ that touches the boundary of $\overline{\mathcal C_{free}}$, there exists another feasible path $\tau'$ in the same homotopy class as $\tau$ such that for all $q \in \tau'$, the $\delta = (\delta_x, \delta_y, \delta_{\theta})$ open ball around $q$ for some arbitrarily small but fixed $\delta_x, \delta_y, \delta_{\theta} > 0$, denoted as $B_{\delta}(q)$, is contained in $\mathcal C_{free}$. We call this the \emph{$\delta$-clearance} property. This suggests $B_{\delta}(\tau(0))$ and $B_{\delta}(\tau(1))$ must themselves be contained in $\mathcal C_{free}$, which is a very mild assumption.

A \emph{cost metric} is needed to compute the cost of feasible paths to measure path optimality. Generally, such a cost for $SE(2)$ depends on weights given to the rotational degree. To that end, we define the cost of a path $\tau$ as 
\begin{align}
\begin{split}
J(\tau) = &\alpha \int_0^1 \sqrt{dx^2 + dy^2} + \beta \int_0^1 |d\theta|, \\
&\alpha, \beta \ge 0, \, \alpha\beta \ne 0, \label{eq:cost}
\end{split}
\end{align}
where $x, y, \theta$ are functions of $t$. In other words, the cost of a path $\tau$ is the weighted sum of the Euclidean path length of $\tau$'s projection to $\mathbb R^2$ and the cumulative rotations. The $\alpha =1$ and $\beta = 0$ case corresponds to ignoring the rotational cost. 
Note that costs other than Eq.~\eqref{eq:cost} may be defined depending on the robot's physical constraints. 

Given two configurations $q_0, q_1 \in \mathcal C_{free}$ between which a feasible path exists, let $\tau*$ be the path with $\tau^*(0) = q_0$ and $\tau^*(1) = q_1$ such that $J(\tau*)$ is minimized, i.e., $J_{q_0, q_1}^* := J(\tau^*) = \inf_{\tau} J(\tau)$. 
This work seeks to compute a path $\tau$ with its cost $J(\tau)$ asymptotically approaching $J_{q_0, q_1}^*$. 
\vspace{-2mm}
\subsection{Visibility and Visibility Graphs}
We briefly introduce concepts surrounding visibility and visibility graphs. For two-dimensional points $p_1, p_2 \in \mathcal W \backslash O$, they are \emph{visible} to each other if $p(t) = tp_1 + (1-t)p_2 \in \mathcal W \backslash O$ for all $t \in [0, 1]$. The definition may be naturally extended to any configuration space, including $SE(2)$: for two configurations $q_1, q_2 \in \mathcal C_{free}$, they are mutually visible if  $q(t) = tq_1 + (1-t)q_2 \in \mathcal C_{free}$ for all $t \in [0, 1]$.
A \emph{visibility graph} $\mathcal{VG}(V, E)$ can then be readily defined based on the definition of visibility, which contains a discrete set of vertices $V \subset \mathcal C_{free}$. For any two vertices $v_1, v_2 \in V$, if they are visible to each other, then there exists a straight edge $(v_1, v_2) \in E$. Conversely, $E$ only contains such edges.

Visibility graphs, properly computed over some free configuration space, greatly facilitate the computation of optimal paths between configurations that are not mutually visible. For computing optimal paths, considering all configuration space features is often unnecessary. For example, in two-dimensional polygonal environments (Fig.~\ref{fig:vg}), only \emph{reflex vertices} (i.e., vertices at which the angle is larger than 180 degrees) are needed, and only bitangents between reflex vertices are needed as the edges. Such visibility graphs are known as \emph{reduced visibility graphs}. 

\begin{figure}[h!]
\centering
\includegraphics[width=0.98\columnwidth]{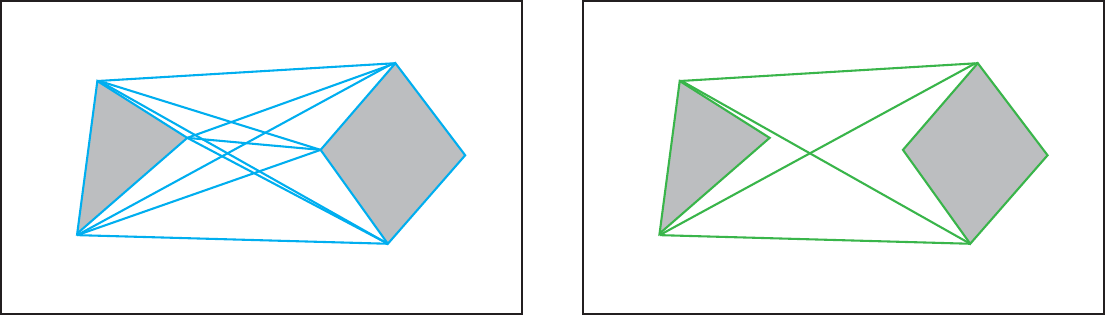}
\caption{[Left] A 2D polygonal environment and the (blue) visibility graph for it. [Right] The (green) reduced visibility graph for the same environment. The additional blue edges do not appear in shortest paths unless the start/goal happens to fall on them (a zero probability event).}
\label{fig:vg}
\end{figure}

\section{Algorithms}\label{sec:algorithm}
For a given (orientation) resolution $n$, \ours slices $SO(2)$ into $n$ layers of uniform thickness, parameterized by feasible rotation interval $\layerrange$. Subsequently, each $\Layer$ is built with its visibility graph $\VG(V, E)$ by Alg.~\ref{alg:building_layers}, allowing navigation with free rotation within $\layerrange$. Each vertex $v(x, y, \layerrange)$ in $V$ represents a feasible subspace of $SE(2)$ consisting of a feasible position in $\mathbb{R}^2$ and a feasible range of rotations $\layerrange$ in $SO(2)$.
To enable the robot to rotate across layers, in Alg.~\ref{alg:propagation}, we propagate vertices in each $\VG$ whose feasible rotation range is larger than $\layerrange$ to neighboring layers by connecting such vertices to the vertices in the $\VG$s located in neighboring layers. After propagation, $\VG$s in all layers are merged into an $\RVG$. With the $\RVG$, shortest paths can be found by any path-finding algorithm, such as A*, given any start and goal pairs.
\vspace{-2mm}
\subsection{Building Layers}
\vspace{-1mm}
\subsubsection{Visibility Queries}
To build the $\VG$ for a layer, for a reflex vertex $v$, all vertices visible to $v$ must be retrieved. We use the \textsc{TriangleExpansion} algorithm~\cite{bungiu2014efficient} to determine the region visible to $v$ given $v$ and $\set{O_i}$.
\subsubsection{Build the Visibility Graph}
Given a robot geometry $P_r$, we can find the bounding polygon $P$ of the rotation range when the robot is rotated to $\theta_{lb}$ and $\theta_{lb}$ respectively. Note that $P$ needs to be an overestimate to ensure the result is collision-free and needs to converge to the true rotation range when the resolution goes to infinity.
%
%
All bounding polygons are merged into one afterward. Then, the Minkowski differences between $P$ and the obstacles $\set{O_i}$ are taken as the grown obstacles $\set{O_i'}$, as shown in Fig~\ref{fig:semi-alge}, after which the $\VG$ is constructed on the vertices of grown obstacles $\set{O_i'}$. The visible regions of all vertices are cached in a dictionary $C$ for future use. Alg.~\ref{alg:building_layers} outlines the whole process.
\begin{figure}[h!]
\centering
\begin{subfigure}[b]{0.24\linewidth}
\includegraphics[width=\linewidth, trim=5.3cm 2.5cm 4.7cm 2.5cm, clip]{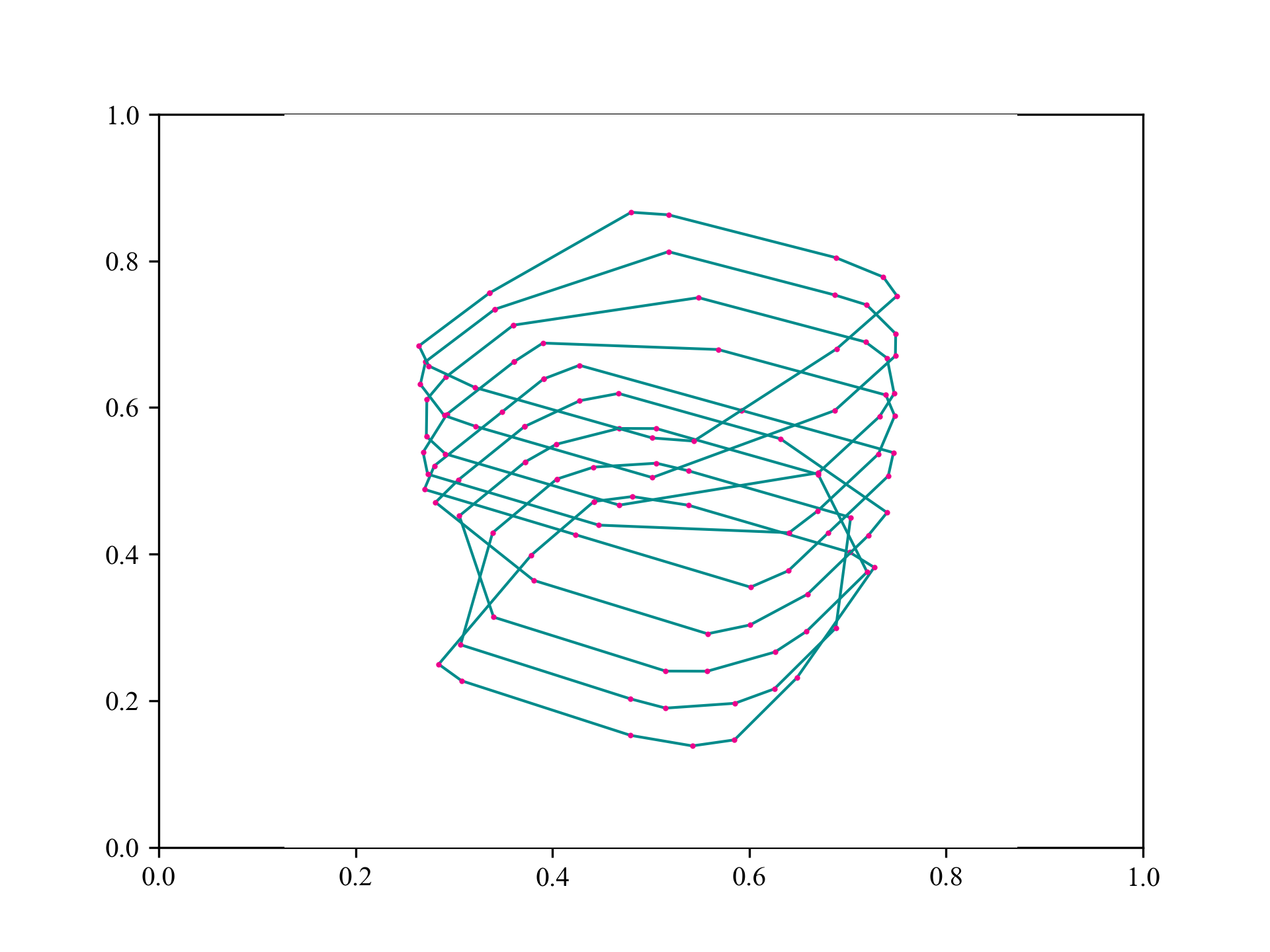}
\caption{}
\end{subfigure}
\begin{subfigure}[b]{0.24\linewidth}
\includegraphics[width=\linewidth, trim=5.3cm 2.5cm 4.7cm 2.5cm, clip]{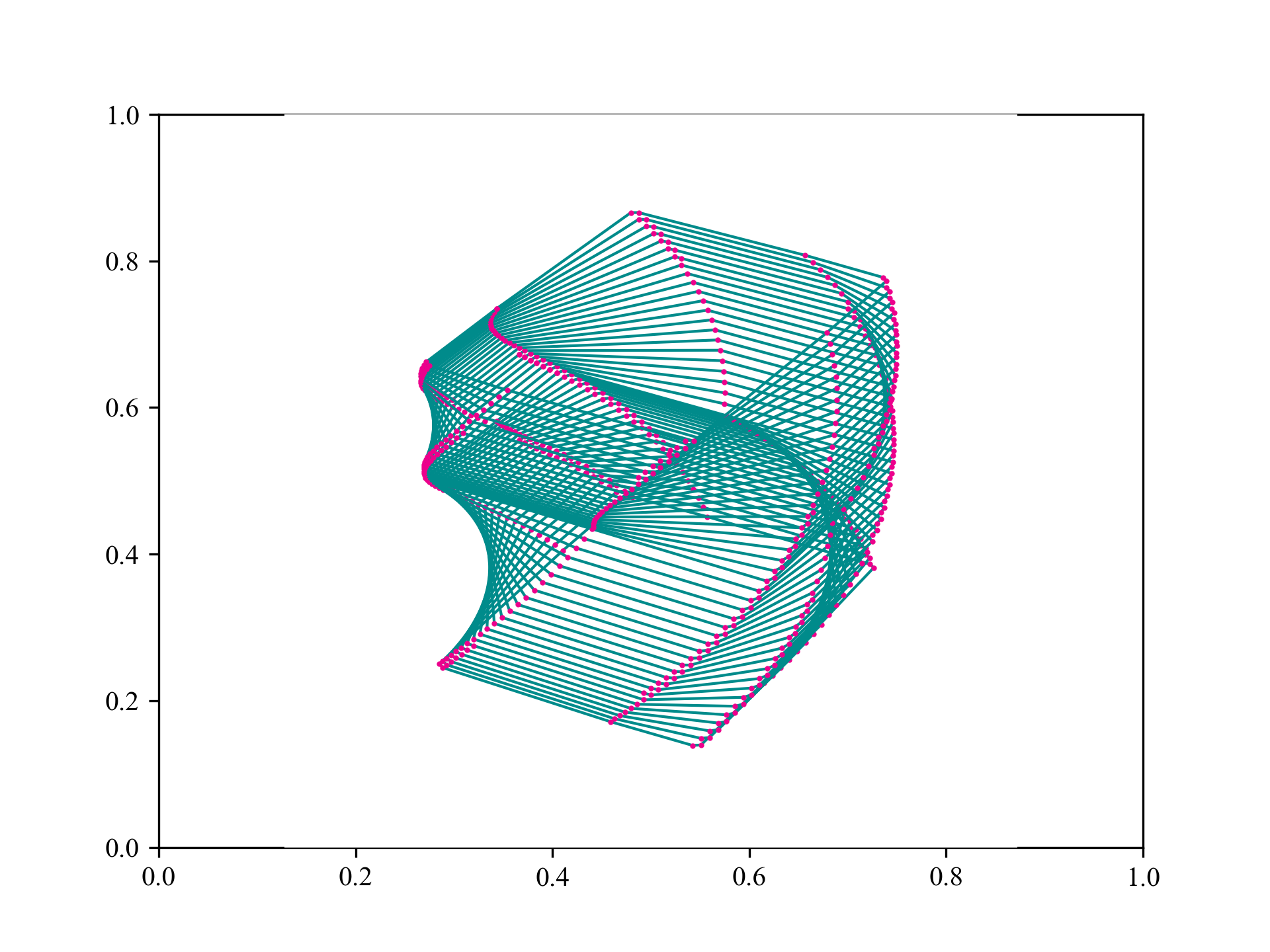}
\caption{}
\end{subfigure}
\begin{subfigure}[b]{0.24\linewidth}
\includegraphics[width=\linewidth, trim=5.3cm 2.5cm 4.7cm 2.5cm, clip]{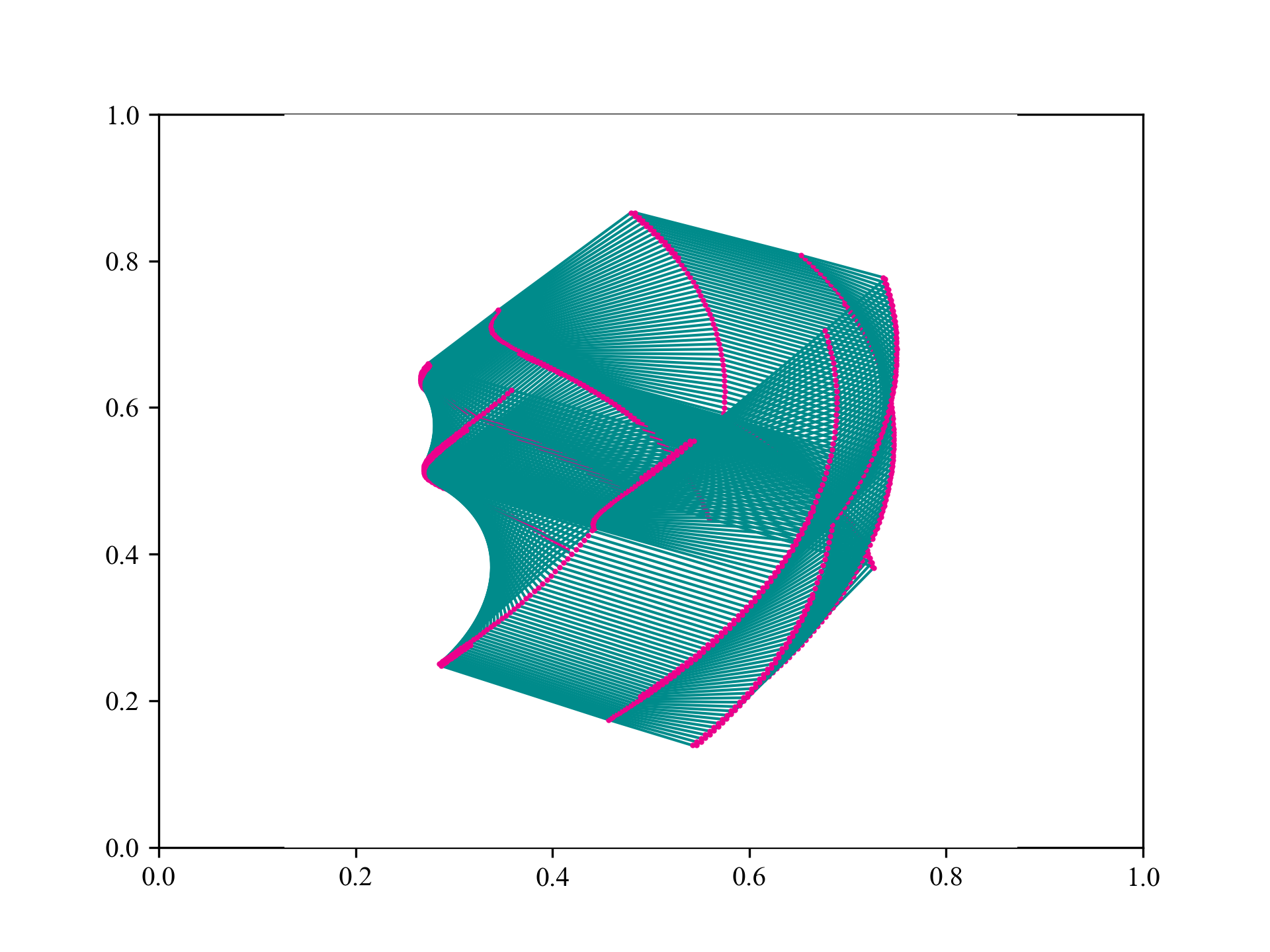}
\caption{}
\end{subfigure}
\begin{subfigure}[b]{0.24\linewidth}
\raisebox{.25\height}{\includegraphics[width=\linewidth, trim=1.22cm 0.9cm 0.8cm 0.5cm, clip]{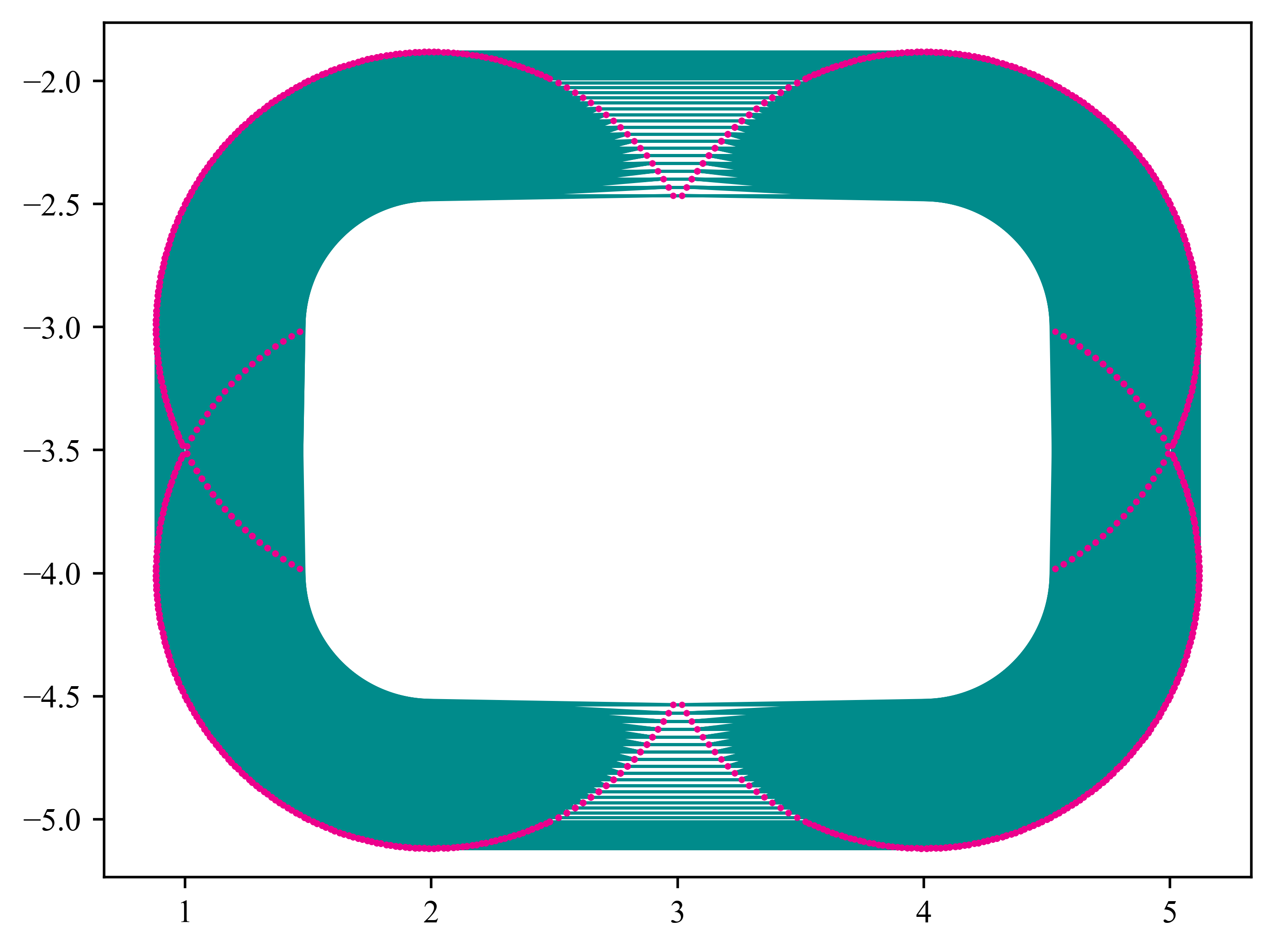}}
\caption{}
\label{fig:semi-alge:d}
\end{subfigure}
\caption{(a)-(c) The semi-algebraic sets composed of the grown obstacles from all layers with resolutions=18, 90, 180 from left to right resulted from 2 rectangles, a robot and an obstacle, shown in Fig~\ref{fig:resolution}. (d) The semi-algebraic sets projected to the $\mathbb{R}^2$ space when resolution=180.}
\label{fig:semi-alge}
\vspace{-4mm}
\end{figure}

\subsection{Vertex Propagation}
\label{subsection:propagation}
In a $\VG$, a random vertex $v$ can rotate within the range $\layerrange$ if it is not inside the grown obstacles. Thus, if a vertex is not located in any of the grown obstacles in the neighboring layers, it can freely rotate in the range of neighboring layers at exactly its own location in $\mathbb{R}^2$.
\begin{algorithm}
\begin{small}
$V=\varnothing$, 
$A = \textsc{TriangleExpansion}(v, \set{O_i})$\\
\For{$O_i\in \set{O_i}$}{
 \For{$v'\in O_i$}{
    \lIf{$v' \in A$}{
        $V = V\cup v'$
    }
 }
}
return $V, A$
\caption{\textsc{VisibilityQuery}($v, \set{O_i}$)} \label{alg:visible_vertices}
\end{small}
\end{algorithm}

\begin{algorithm}
\begin{small}
\caption{\textsc{BuildLayer}($P_r, \layerrange, \set{O_i}$)} \label{alg:building_layers}
  find the bounding polygon $P$ when $P_r$ is rotated to $\theta_{lb}$ and $\theta_{ub}$ respectively\\
     move $P$ to the origin, invert $P$\\
     $\{O_i'\} = \{P\ominus O_i\} $ (Minkowski difference)\\
     merge any obstacles in $\{O_i'\}$ that have intersections\\
     $V=\varnothing, E=\varnothing, C=\varnothing$\\
    \For{$O_i'\in \set{O_i'}$}
    {
        \For{$v \in O_i'$}
        {
            \If{\textsc{isReflex}($v$)}{
                 $V = V\cup \set{v}$\\
                 $V_{visible}, A = \textsc{VisibleQuery}(v, \set{O_i'})$\\
                 $C$.insert$(\set{v, A})$\\
                \For{$v'\in V_{visible}$ and \textsc{isReflex}($v'$)} 
                {
                    \If{\textsc{isBitangent}($v, v'$)}{
                         $e = (v, v')$\\
                         $V = V\cup \set{v'}$, $E = E\cup \set{e}$
                    }
                }
            }
        }
    }
     return $\Layer(\VG(V,E), C)$ 
\end{small}
\end{algorithm}
For example, Fig~\ref{fig:minkowski_sum} shows parts of the Minkowski sums (black and brown solid lines) of the {bounding polygon of the} robot at two different poses (as dashed rectangles) that are $10$ degrees apart with an obstacle (its red border is partially shown). 
In the figure, vertices $A', B$, and $B'$ can rotate in both layers.
{To merge $\VG$s across the layers, for each vertex $v(x,y, \layerrange)$ in the current layer, the connection with another vertex $v'(x', y', [\theta_{lb}', \theta_{ub}'])$ from the neighboring layers will be established if: (1) $v'$ falls into the visible area of $v$, (2) The edge $(v, v')$ is bitangent,} e.g., vertex $A'$ and $B$ as well as vertex $B$ and $B'$ in Fig~\ref{fig:minkowski_sum} can be connected. 
However, if $v$ and $v'$ are directly connected by a line segment, the translation and rotation are coupled, meaning there is no guarantee of collision-free rotation at positions other than $(x', y')$. To address this, we decouple translation and rotation by introducing a substitute vertex $v''(x', y', \layerrange)$ for $v'$. A two-fold connection between $v$ and $v'$ is then established by connecting $(v, v'')$ and $(v', v'')$, while ensuring that the total connection cost remains the same as the direct connection.

\begin{figure}[h!]
  \begin{center}
    \includegraphics[width=0.225\textwidth]{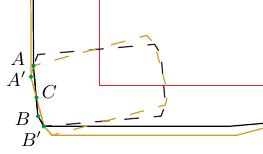}
  \end{center}
\vspace{-3mm}
  \caption{\label{fig:minkowski_sum}Intersections of Minkowski sum boundaries (solid black and brown lines) of the robot at two poses (dashed lines) and an obstacle (partial border showing in red).}
\end{figure}
Such propagation can be divided into two passes - a \emph{forward pass} and a \emph{backward pass}. The  \emph{forward pass} iteratively propagates vertices from $\Layer_i$ to $\Layer_{i+1}$, and the \emph{backward pass} propagates vertices from $\Layer_{i+1}$ to $\Layer_i$. In this way, each vertex can be propagated as far as possible. 
To avoid redundant propagation in the backward pass, we use a dictionary $D[\ldots]$ to cache information from the forward process. 
Each key is a layer index $i$, mapping to a set of tuples $(i', j)$, where $i'$ is the layer where the vertex was generated, and $j$ is its index in that layer. In the \emph{backward pass}, if $i' < i$, the vertex originates from a lower, already-checked layer, making further checking unnecessary.

\begin{algorithm}
\def\next{(l+1)\mod n}
\def\prev{(l-1)\mod n}
\begin{small}
\vspace{0.025in}
     $V = \bigcup_{i=1\cdots n} V_i$, $E = \bigcup_{i=1\cdots n}E_i$, $C=\bigcup_{i=1\cdots n}C_i$\\
     $D=\set{1:\set{(1, 0), \cdots, (1, \mid V_1\mid)}, \cdots, n:\set{(n, 0), \cdots, (n, \mid V_n\mid)}}$\\
    \Comment{The forward pass}
    \For{$l$ in $1 \cdots n$}  
    {
        $next=\next$\\
        \For{$(i, j)$ in $D[l]$}
        {
             $v = V_{i}[j]$\\
            \For{$(i', j')$ in $D[next]$}
            {
                 $v'=V_{i'}(j')$\\
                \If{\textsc{isBitangent}$(v, v')$ and $v\in C[v']$}{
                     $v''=(v.x, v.y, \layerrange_{(next)})$\\
                     $e_0=(v, v''), e_1=(v'', v')$\\
                     $V = V\cup v''$, $E = E \cup \set{e_0, e_1}$\\
                     $D[l] = D[l] \cup \set{(i', j')}$\\
                     $D[next]=D[next]\cup \set{(i, j)}$\\
                }
            }
        }
    }
    \Comment{The backward pass}
    \For{$l$ in $n \cdots 1$} {
        $prev=\prev$\\
        \For{$(i, j)$ in $D[l]$}
        {
             $v = V_{i}[j]$\\
            \For{$(i', j')$ in $D[prev]$}
            {
                \lIf{$i' < l$}{continue} 
                 $v'=V_{i'}(j')$\\
                \If{\textsc{isBitangent}$(v, v')$ and $v\in C[v']$}{
                     $v''=(v.x, v.y, \layerrange_{(prev)})$\\
                     $e_0=(v, v''), e_1=(v'', v')$\\
                     $V=V\cup v''$, $E = E \cup \set{e_0, e_1}$\\
                     $D[l] = D[l] \cup \set{(i', j')}$\\
                     $D[prev]=D[prev]\cup \set{(i, j)}$\\
                }
            }
        }
    }
     return $\RVG(V, E)$\\
\vspace{0.025in}
\caption{\small{\textsc{VertexPropagation}($n$ $\Layer$s)}} \label{alg:propagation}
\end{small}
\end{algorithm}
\vspace{-3mm}
\subsection{Searching for Optimal Paths on $\RVG$}
Given the start $s(x_s, y_s, [\theta_{lb}^s, \theta_{ub}^s])$ and the goal $g(x_g, y_g, [\theta_{lb}^g, \theta_{ub}^g])$, Alg.~\ref{alg:add_start_and_goal} is used to add $s$ and $g$ into the $\RVG$ by connecting the visible vertices in each layer that has overlap with $[\theta_{lb}^s, \theta_{ub}^s]$ and $[\theta_{lb}^g, \theta_{ub}^g]$.
In practice, the start and goal generally have a fixed rotation, where we can simply set $\theta_{lb}^s=\theta_{ub}^s$ and $\theta_{lb}^g=\theta_{ub}^g$.
Then, an optimal solution can be found according to the cost metric Eq.~\eqref{eq:cost}. During the search, each vertex $v$'s rotation is represented by the mean value of its $\layerrange$. Since building $\RVG$ doesn't rely on the start and the goal, the same $\RVG$ can be used for multiple queries.
\begin{algorithm}
\begin{small}
\vspace{0.025in}
    \For{$l$ in $1\cdots n$}
    {
        \If{$[\theta_{lb}^s, \theta_{ub}^s]$ overlaps with $[\theta_{lb}^l, \theta_{ub}^l]$}
        {
             $V = V \cup \set{s}$ if $s$ not added\\
             $V_{visible}, A = \textsc{VisibleQuery}(s, O_i)$\\
            \lFor{$(i, j) \in D[l]$}{$v=V_i(j), E = E \cup \set{(s, v)}$}
        }
        \If{$[\theta_{lb}^g, \theta_{ub}^g]$ overlaps with $[\theta_{lb}^l, \theta_{ub}^l]$}
        {
             $V = V \cup \set{g}$ if $g$ not added\\
             $V_{visible}, A = \textsc{VisibleQuery}(q, O_i)$\\
            \lFor{$(i, j)\in D[l]$}{$v=V_i(j), E = E \cup \set{(g, v)}$}
        }
    }
\vspace{0.025in}
\caption{\small{\textsc{AddStartGoal}($s, g$)}} \label{alg:add_start_and_goal}
\end{small}
\end{algorithm}


\subsection{Formal Guarantees}
Due to page limits, we state the key properties of \ours here. The full proofs and properties of \ours will be detailed in an extended journal version.
%
Let $n$ denote the resolution, $m$ the total number of vertices/edges, 
and $k$ denote the number of obstacles,
 including the obstacles and the robot. \ours has a $O(nm^2k)$ complexity. 
\vspace{-2mm}
\begin{theorem}
Assuming $\delta$-clearance, \ours is resolution complete. Moreover, \ours computes asymptotically optimal solutions, if the overestimate of the rotation range converges to the true rotation range as the rotational resolution goes to infinity.
\end{theorem}
\begin{proof}[Proof sketch]
\vspace{-3mm}
An optimal solution consists of connections between bitangent reflex vertices on the surface of obstacles in $SE(2)$, where there are infinite ways to connect the vertices if they have the same lowest cost.
To see that \ours is asymptotically optimal, under the $\delta$-clearance assumption, \ours builds an $\RVG$ that, at sufficiently high resolution, can approximate all the reflex vertices if the overestimate of the rotation range converges to the true rotation range. Any two bitangent reflex vertices can be connected in $\RVG$ with the equivalent cost.
Resolution completeness follows asymptotic optimality. 
\end{proof}


\section{Computationl Evaluation}\label{sec:evaluation}
%

We developed \ours with the hope that it would become a practical tool. As such, much effort has been devoted to efficiently implementing \ours in C++, to be open-sourced. \ours leverages CGAL \cite{cgal:eb-24a} for two-dimensional geometrical computations and polygon operations. Computational evaluations were performed on an Intel i9-14900K CPU with a single CPU core/thread. 

Qualitative evaluations are first provided, highlighting \ours's behavior as the resolution and the cost structure change. These qualitative results corroborate \ours's the correctness of \ours. 
Then, extensive quantitative evaluations, including full \ours performance characterization under different resolutions and performance comparisons with state-of-the-art sampling-based algorithms, confirm the superior performance of \ours, on both computational speed and solution quality fronts. Given limited pages, we mainly work with the cost metric where $\alpha = 1$ and $\beta = 0$ in Eq.~\eqref{eq:cost}. Note that even when we set $\beta = 0$, it is generally undesirable for the robot to execute unnecessary rotations because doing so will often add to the path's total length. 

\subsection{Qualitative Behavior of \ours}
\subsubsection{Effects of Varying Resolutions} \ours is an \emph{anytime} algorithm, which finds a solution quickly. As the resolution increases, the solution quality improves. The experiment illustrated in Fig.~\ref{fig:resolution} demonstrates \ours's such behavior. In the experiment, a rectangular robot is asked to go from the left to the right. At $n=36$, the robot can only find a path above the large obstacle with $0.031s$ build time and $0.006s$ search time. As the resolution increases to $n=72$, the robot finds a shorter path going through the narrower passage underneath the obstacle with $0.085s$ build time and $0.021s$ search time. The corresponding $\RVG$s (projected, without the orientation dimension) are also shown.
\begin{figure}[h!]
\centering
\includegraphics[width=0.32\linewidth, trim=4.08cm 6.28cm 2.2cm 5.12cm, clip]{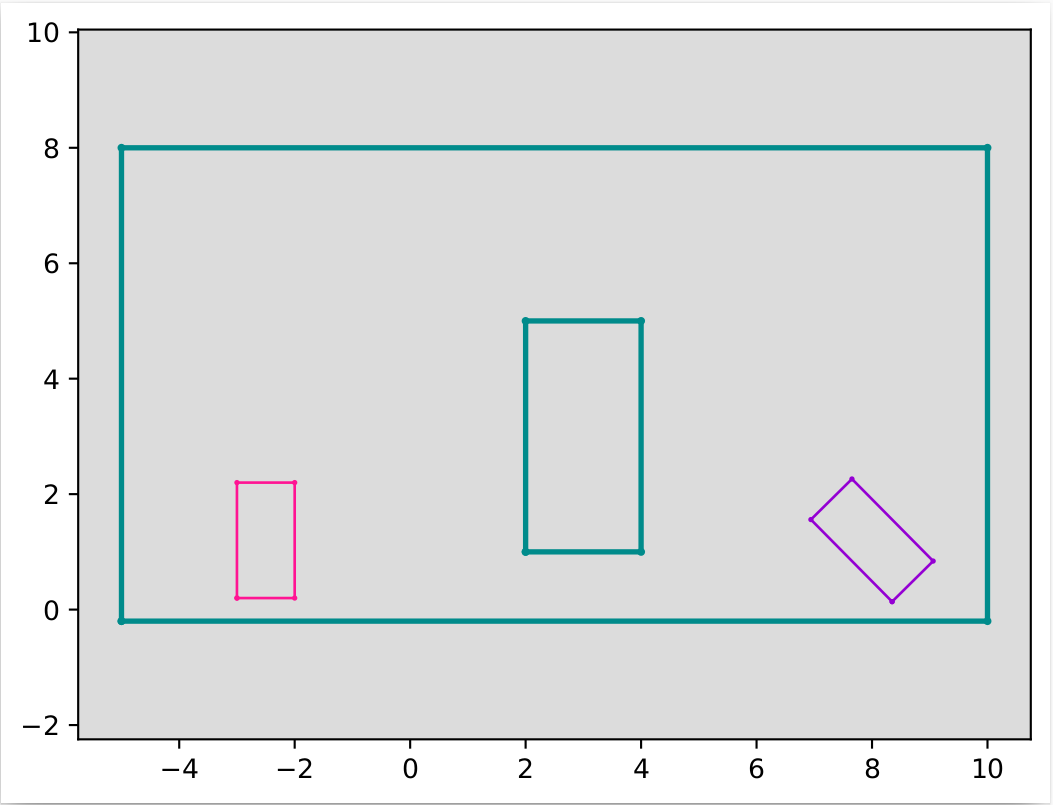}
\includegraphics[width=0.32\linewidth, trim=4.08cm 6.28cm 2.2cm 5.12cm, clip]{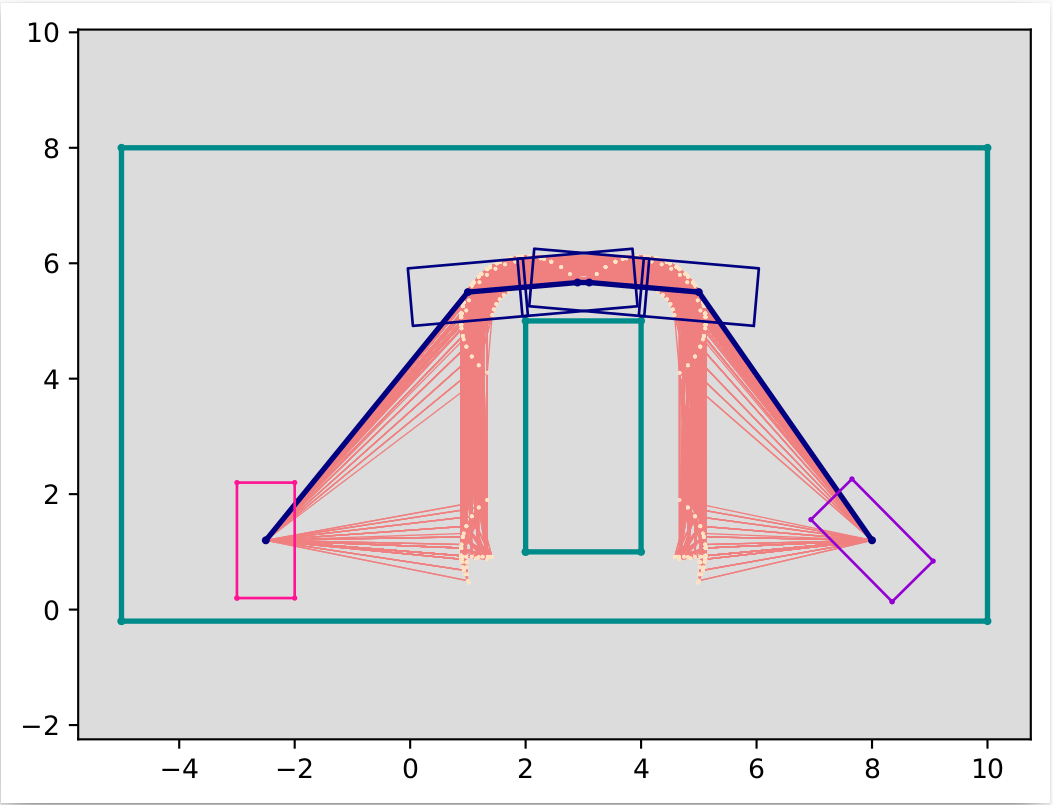}
\includegraphics[width=0.32\linewidth, trim=4.08cm 6.28cm 2.2cm 5.12cm, clip]{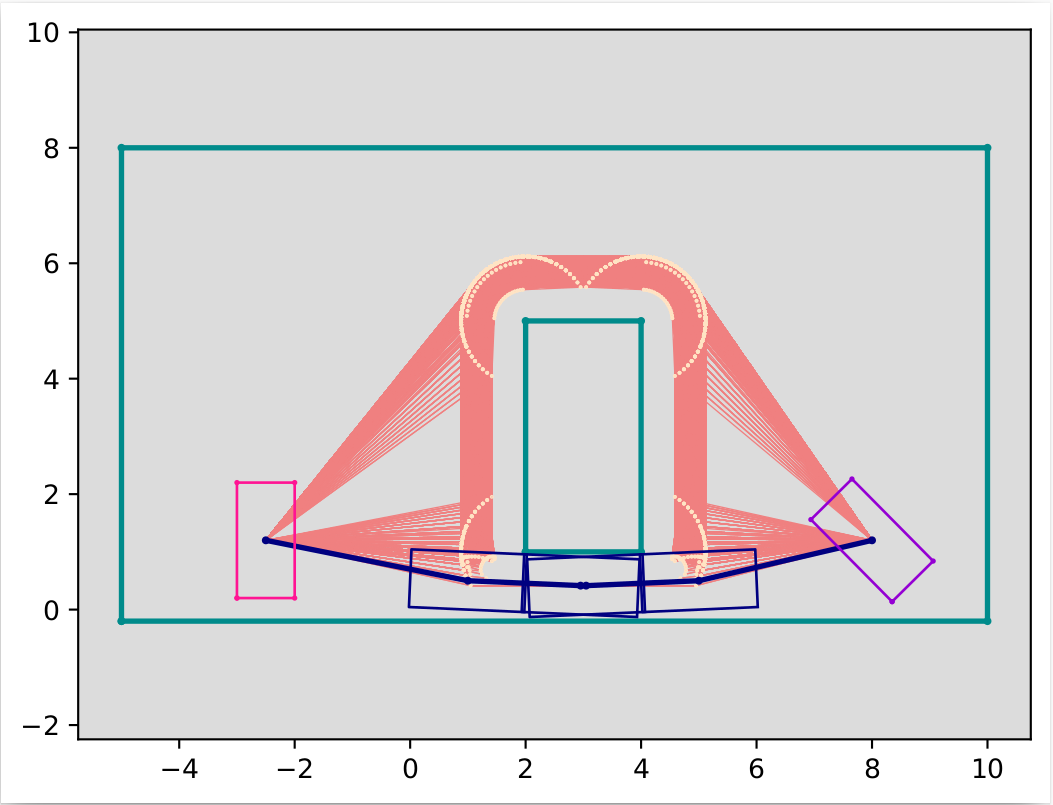}
\caption{The solutions from our algorithm with different resolutions. The pink, purple, and deep cyan rectangles represent the start configuration, goal configuration, and obstacles(including boundaries). [Left] The map. [Middle] The solution found at resolution $n=36$. [Right] The solution at resolution $n=72$, taking more time to compute.}
\label{fig:resolution}
\end{figure}

\subsubsection{Effects of Varying Cost Metrics} To be versatile, \ours should readily support different cost structures, allowing $\alpha$ and $\beta$ in Eq.~\eqref{eq:cost} to change freely. To showcase this capability, we carefully design the environment in Fig.~\ref{fig:cost-structure}. At resolution $n=36$, results with four $\alpha$ and $\beta$ combinations are given. On the top left, $\alpha = 1$ and $\beta = 0$. This forces the robot to select the shortest Euclidean distance path (in dark blue) of length $106.77$ even though the path has a large rotation cost of $48.00$ radians. When $\alpha = \beta = 0.5$ as on the top right, a path with a longer Euclidean distance of $116.34$ is chosen, with a total rotation of $18.33$ radians. The pattern continues. In the end, at $\alpha = 0$ and $\beta = 1$, the path with the least rotation (0.17 radian\footnote{The actual rotation is zero; we are conservatively reporting the range of each orientation slice. In this case it is $2\pi/36 \approx 0.17$.}) and the largest Euclidean distance ($145.69$) gets selected. On average, \ours takes $2.087s$ and $0.045s$ to build $\RVG$ and search for the optimal path. Qualitatively, \ours correctly computes the desired shortest path based on the provided cost metric. 
\begin{figure}[h!]
\vspace{2mm}
\centering
\includegraphics[width=0.45\linewidth, trim=1.452cm 1.22cm 0.649cm 2.02cm, clip]{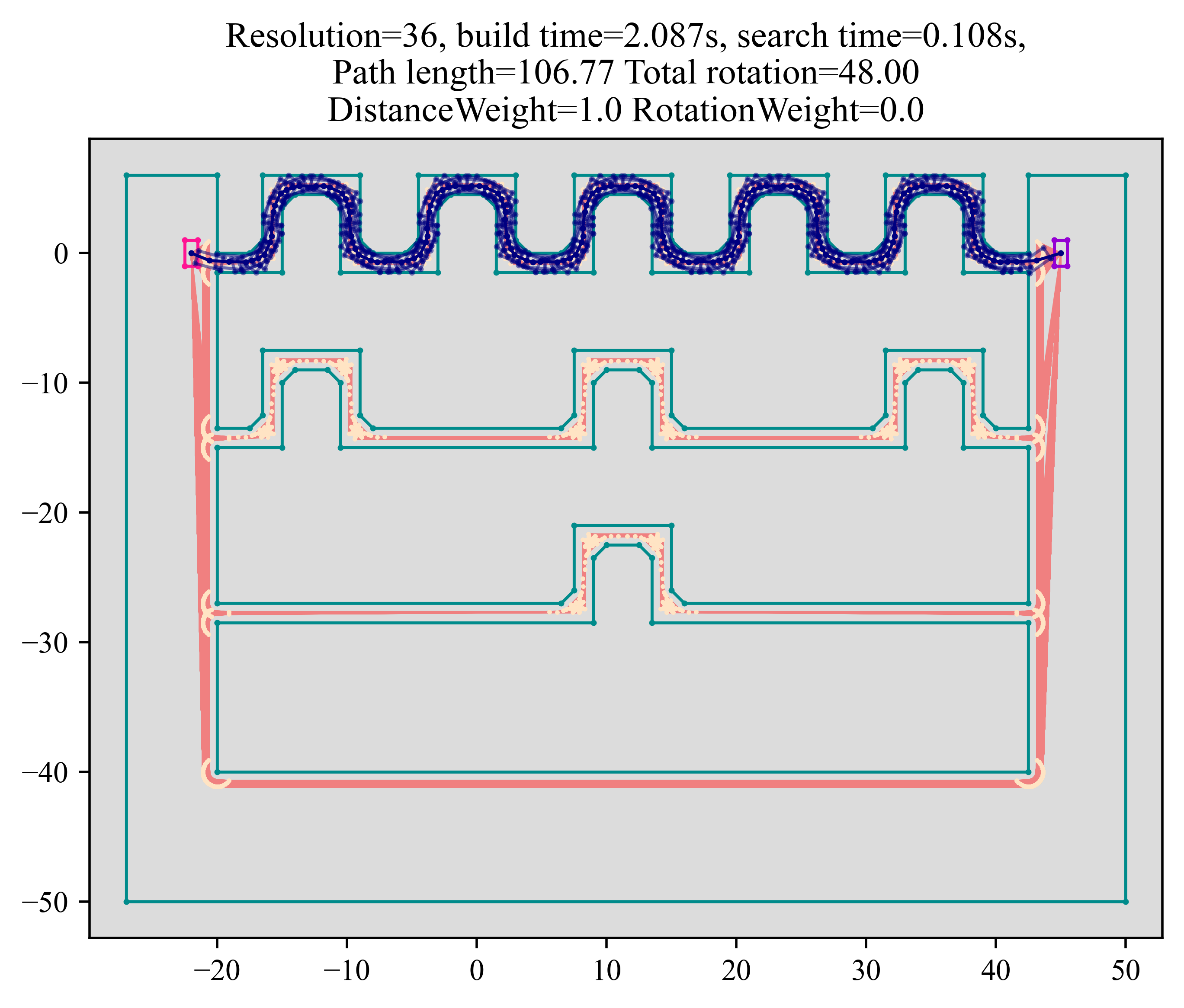}
\includegraphics[width=0.45\linewidth, trim=1.452cm 1.22cm 0.649cm 2.02cm, clip]{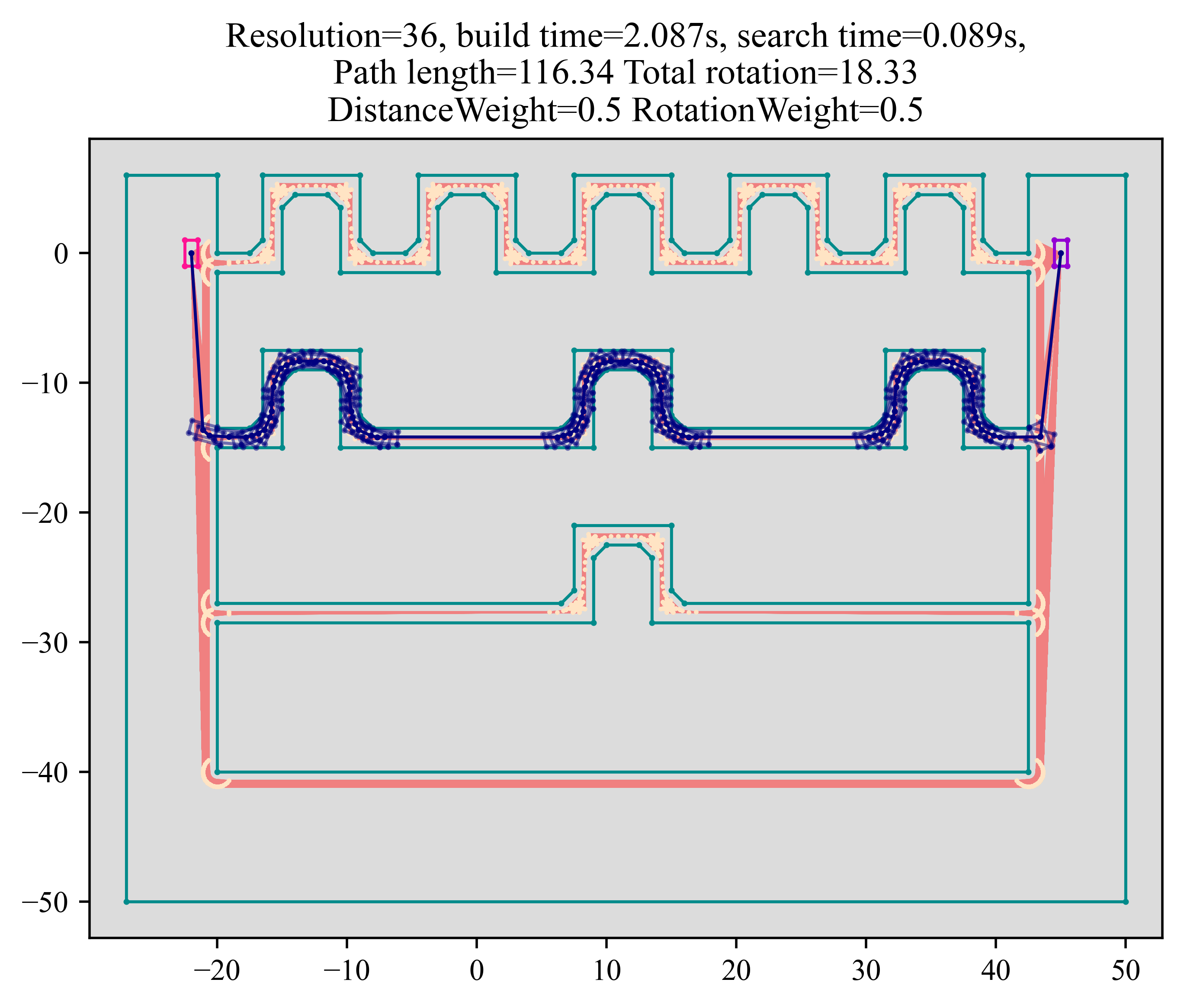} \\
\vspace{1mm}
\hspace{0.001\linewidth}
\includegraphics[width=0.45\linewidth, trim=1.452cm 1.22cm 0.649cm 2.02cm, clip]{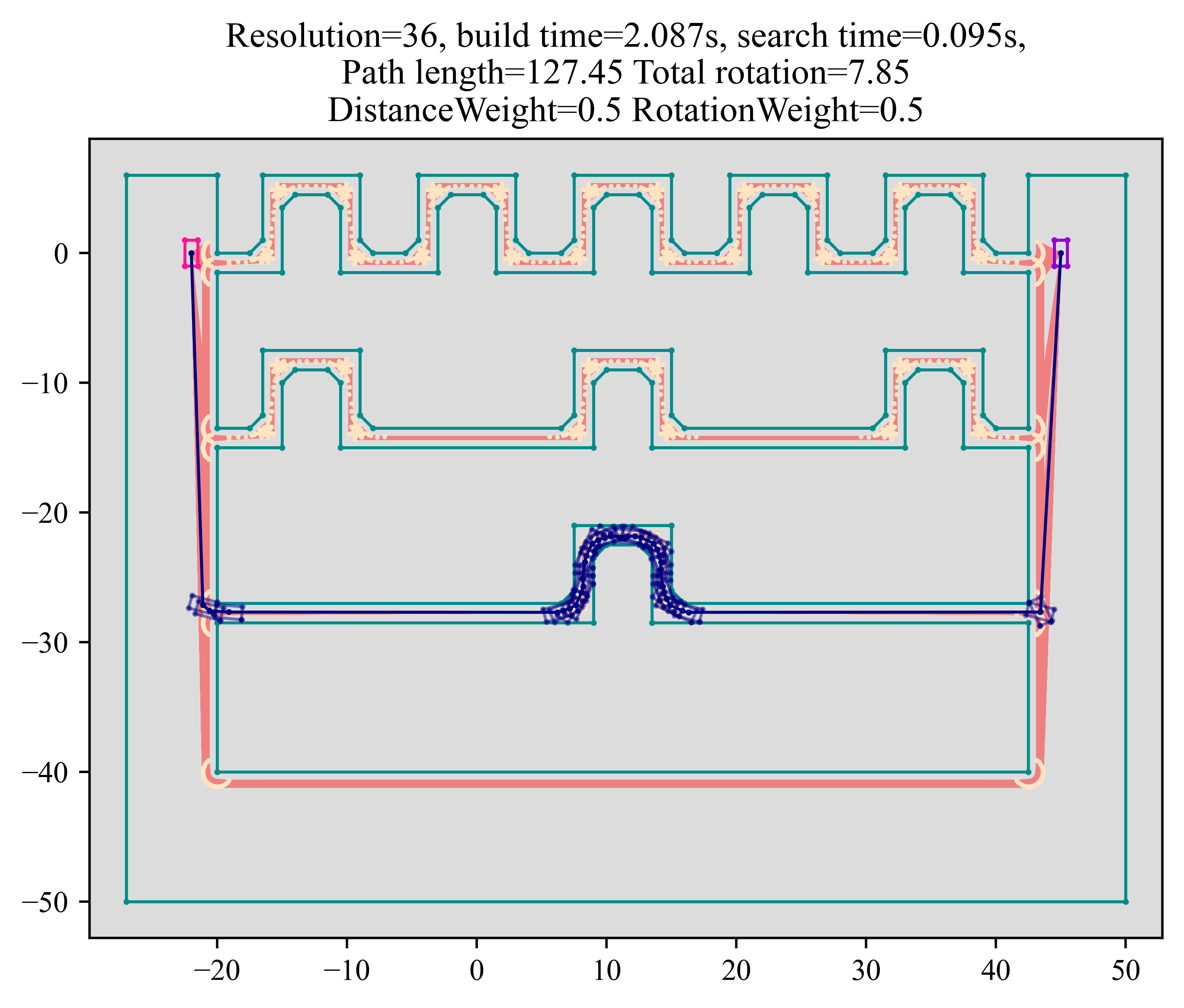}
\includegraphics[width=0.45\linewidth, trim=1.452cm 1.22cm 0.649cm 2.02cm, clip]{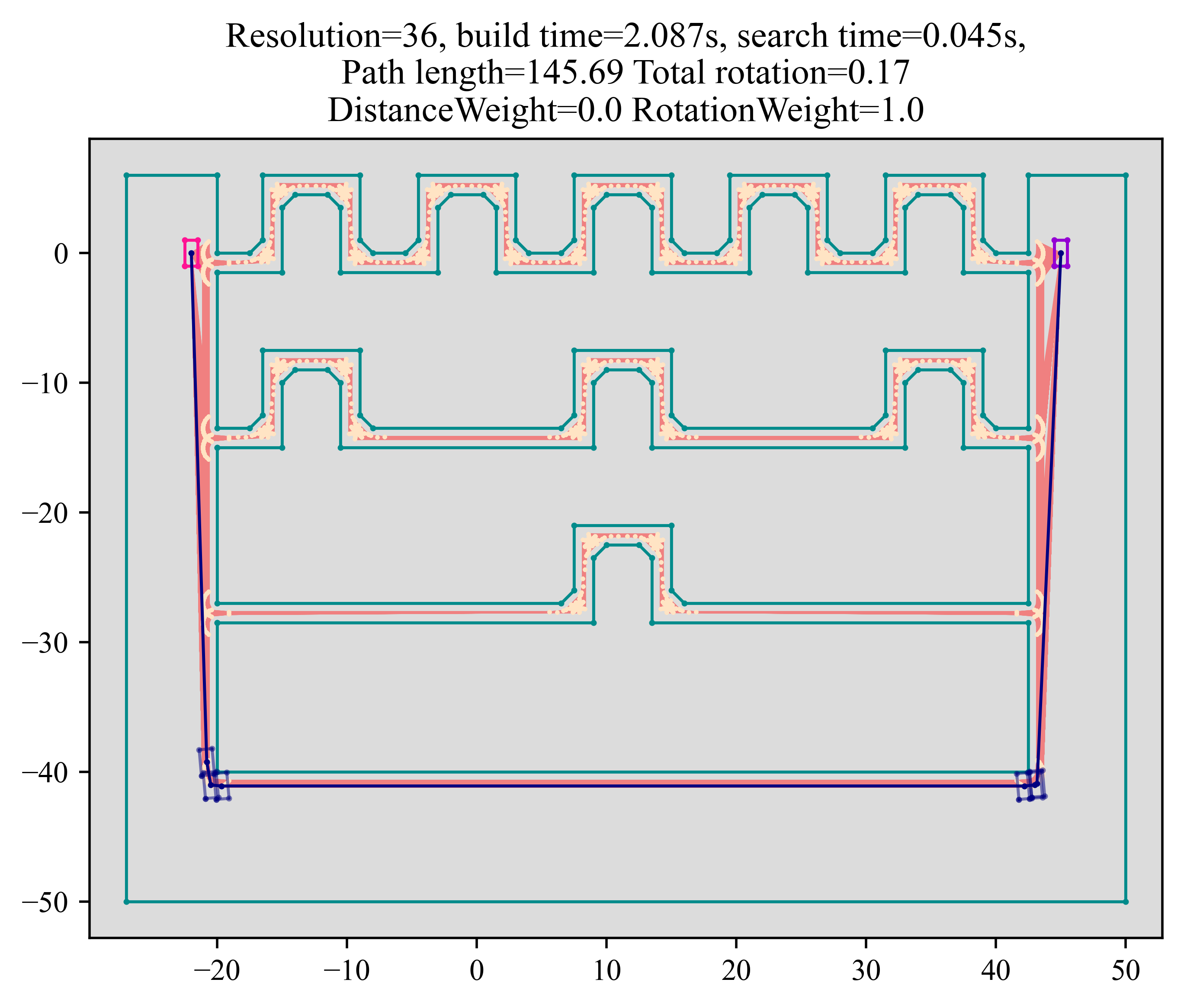}
\vspace{1mm}
\caption{The outcome of running \ours at resolution $36$ with different $\alpha$ and $\beta$ values. Both $\RVG$ and the shortest path (in blue) are shown. [Top Left] $\alpha = 1$ and $\beta = 0$. 
[Top Right] $\alpha = 0.5$ and $\beta = 0.5$. 
[Bottom Left] $\alpha = 0.48$ and $\beta = 0.52$. 
[Bottom Right] $\alpha = 0$ and $\beta = 1$.}
\label{fig:cost-structure}
\end{figure}
\begin{figure*}[!t]
\centering
\includegraphics[width=0.095\linewidth, trim=3.05cm 1.23cm 2.25cm 0.66cm, clip]{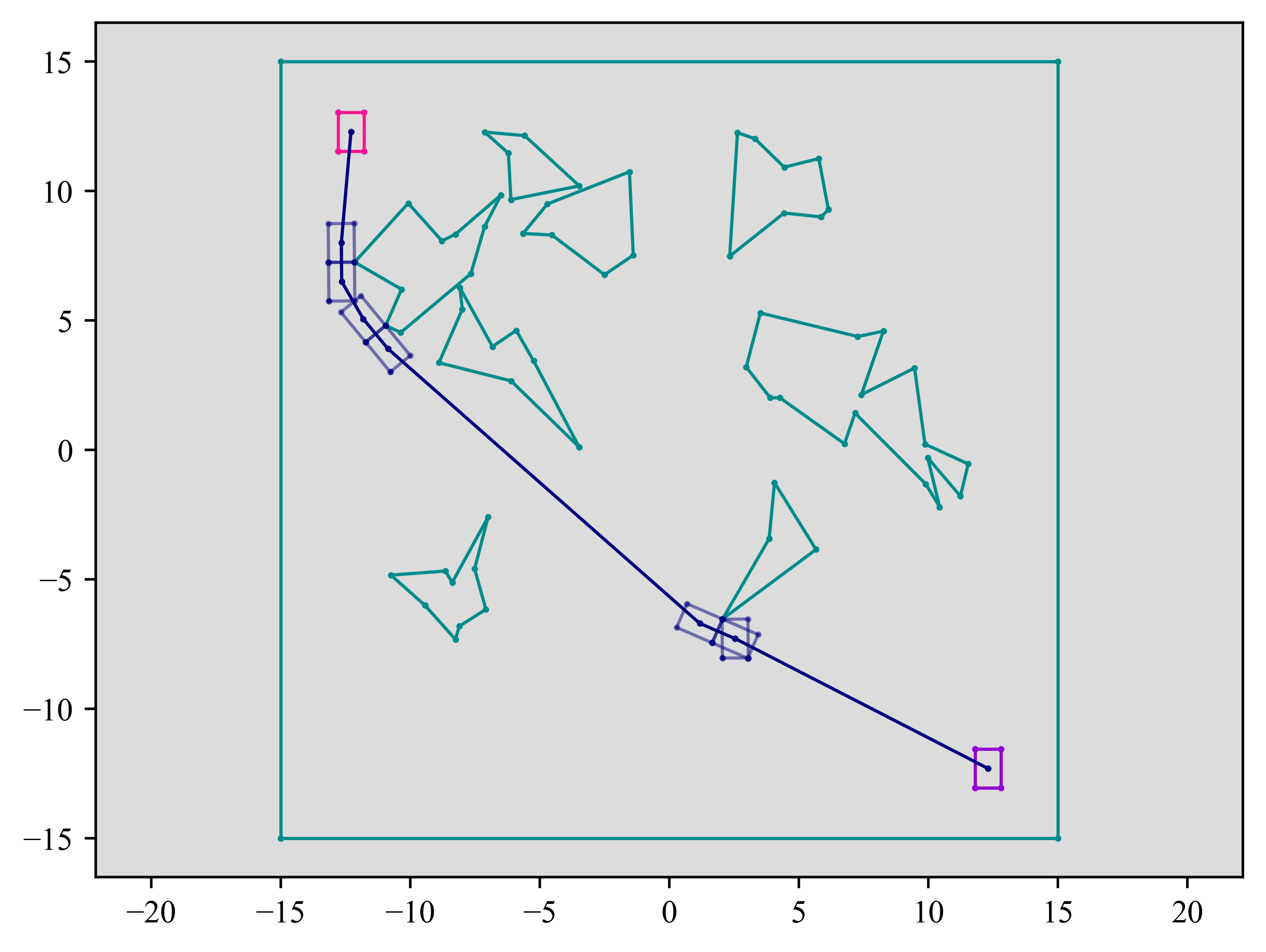}\hfill
\includegraphics[width=0.095\linewidth, trim=3.05cm 1.23cm 2.25cm 0.66cm, clip]{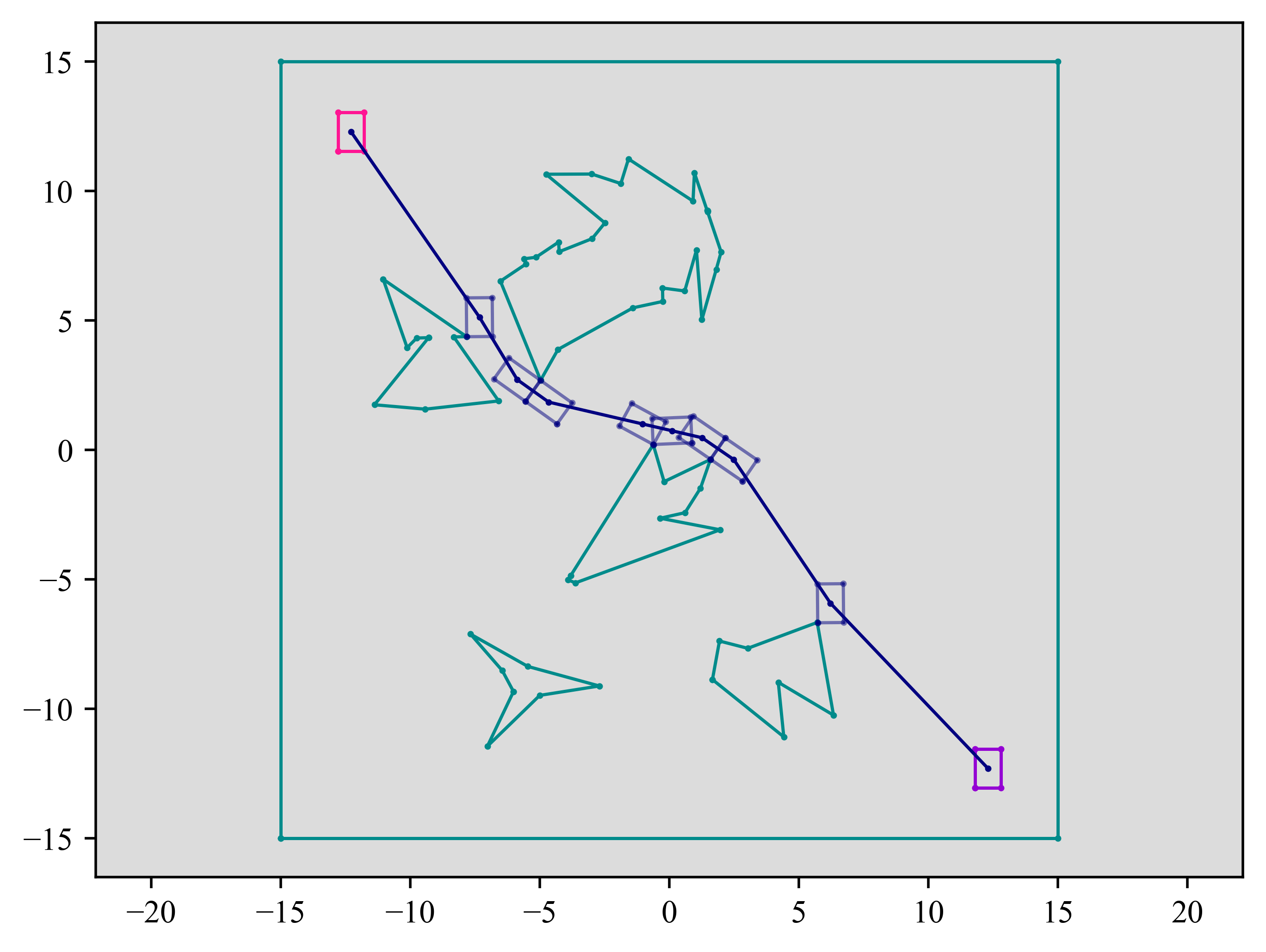}\hfill
\includegraphics[width=0.095\linewidth, trim=3.05cm 1.23cm 2.25cm 0.66cm, clip]{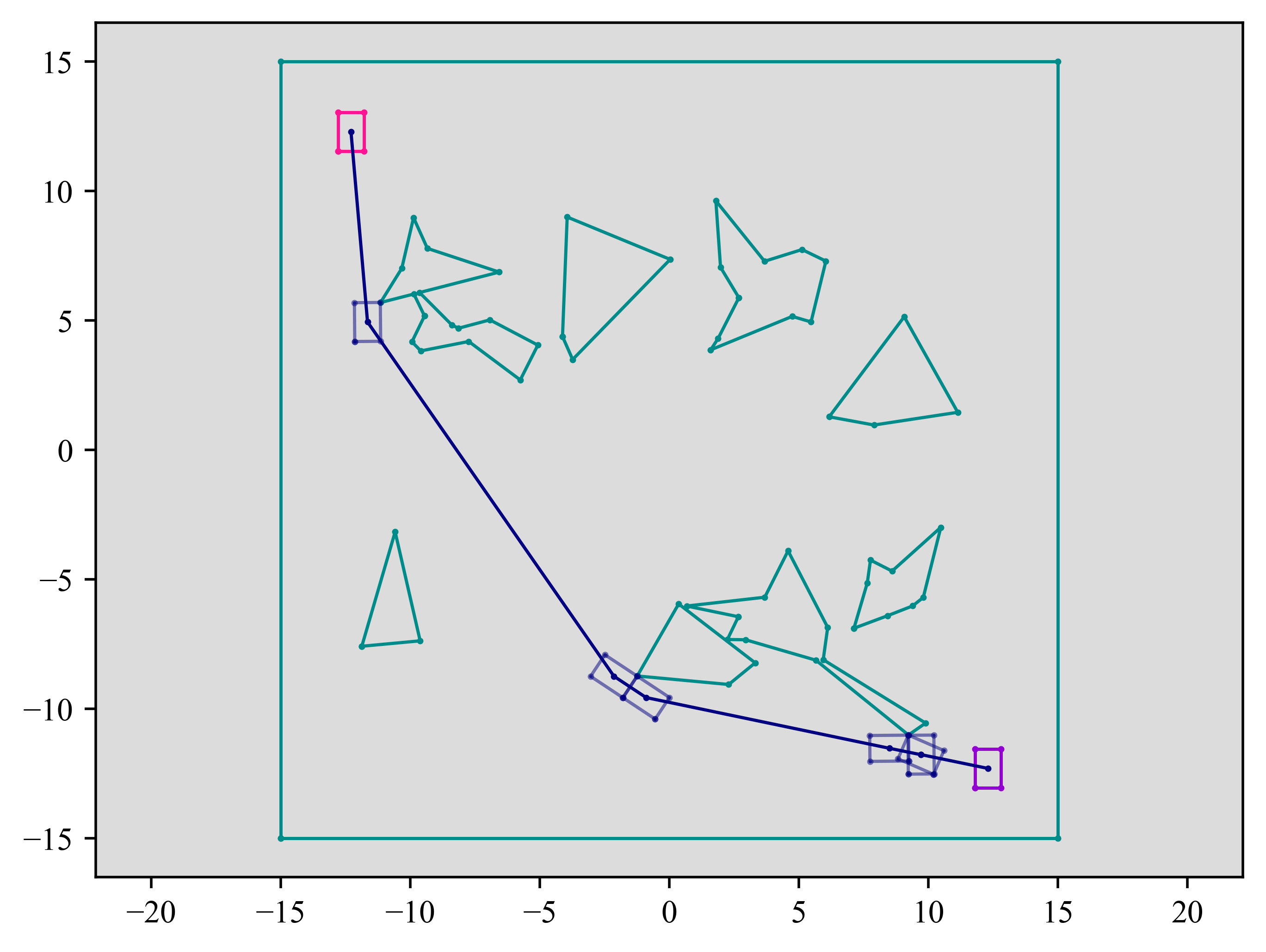}\hfill
\includegraphics[width=0.095\linewidth, trim=3.05cm 1.23cm 2.25cm 0.66cm, clip]{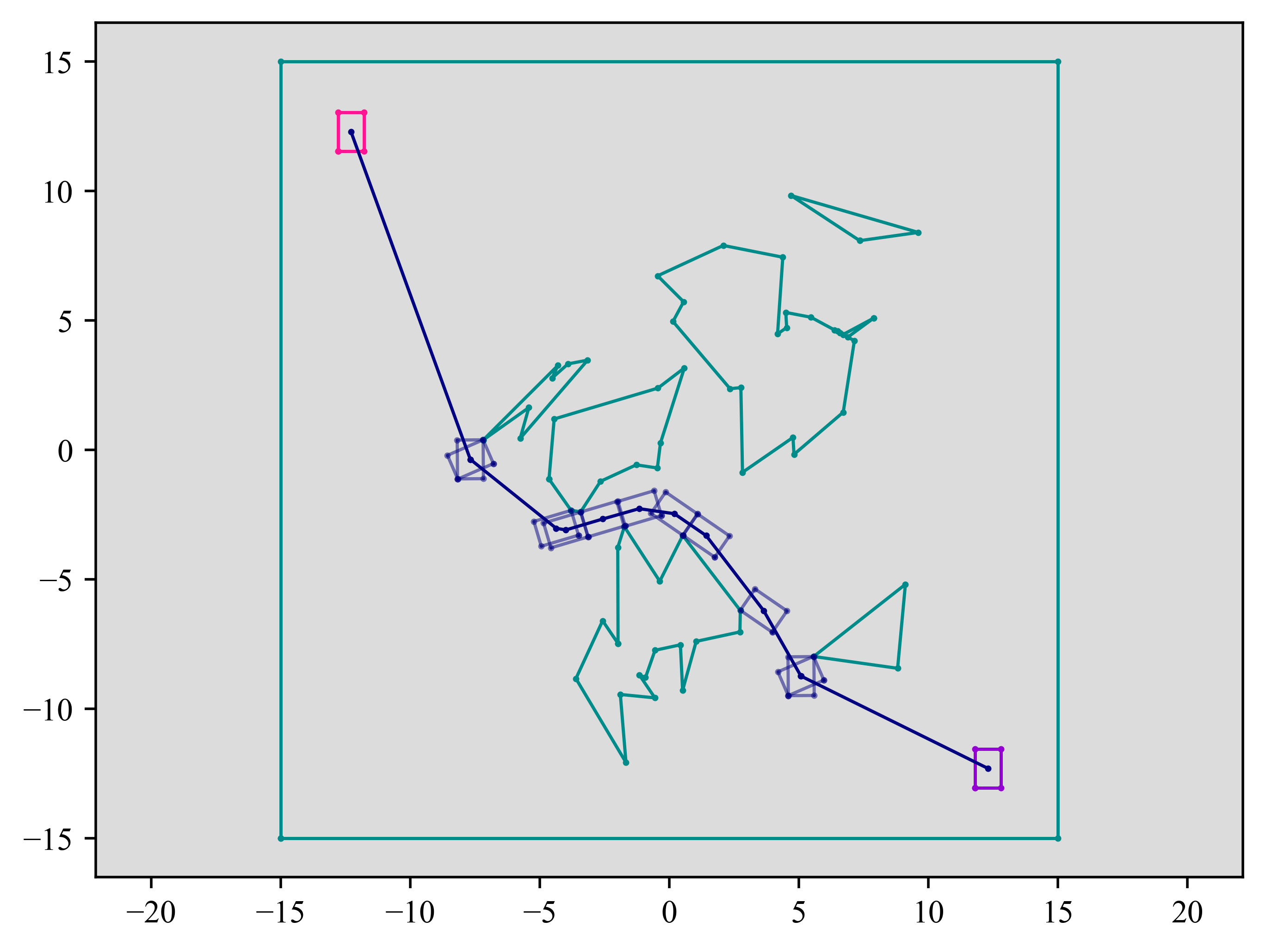}\hfill
\includegraphics[width=0.095\linewidth, trim=3.05cm 1.23cm 2.25cm 0.66cm, clip]{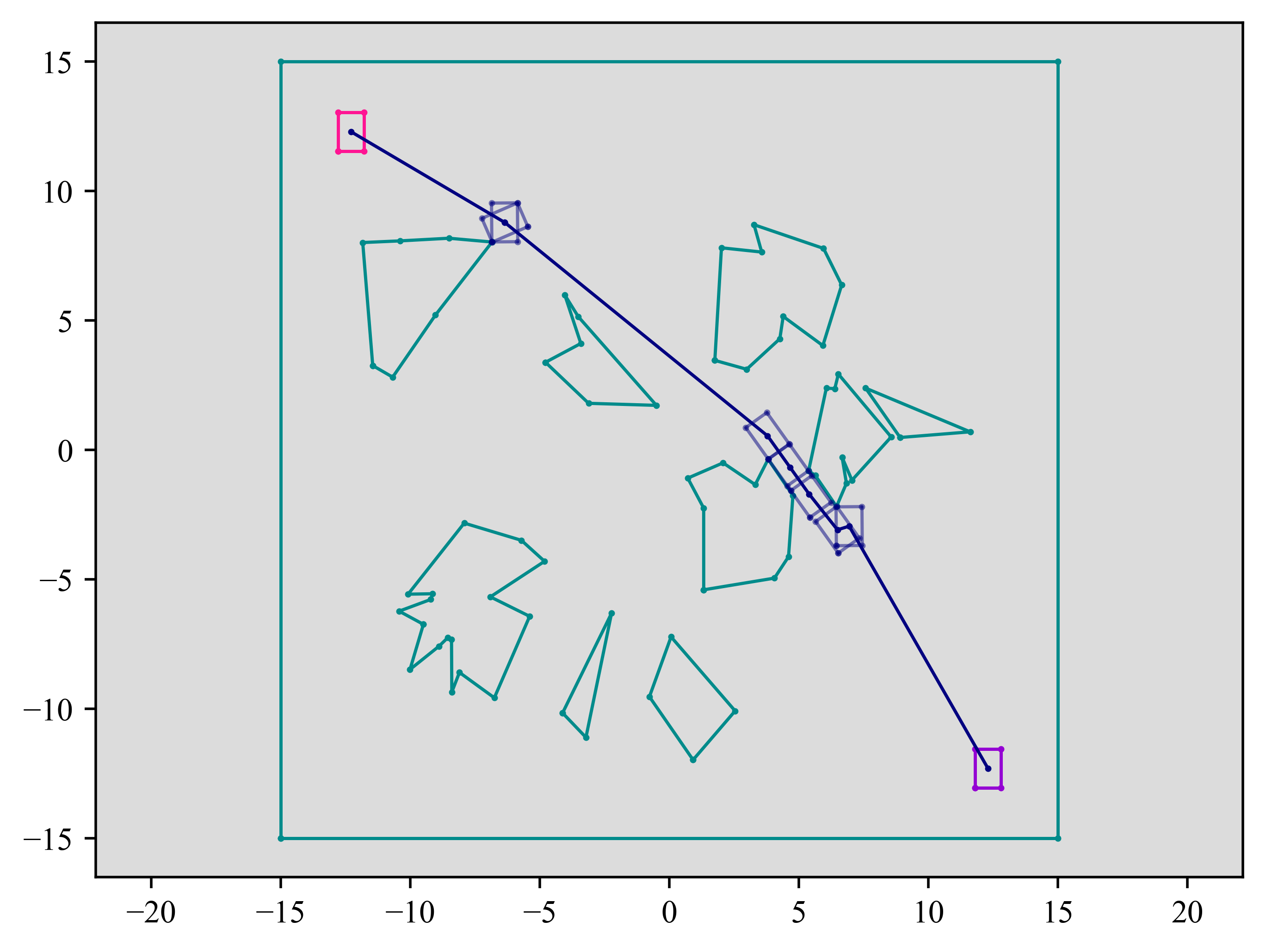}\hfill
\includegraphics[width=0.095\linewidth, trim=3.05cm 1.23cm 2.25cm 0.66cm, clip]{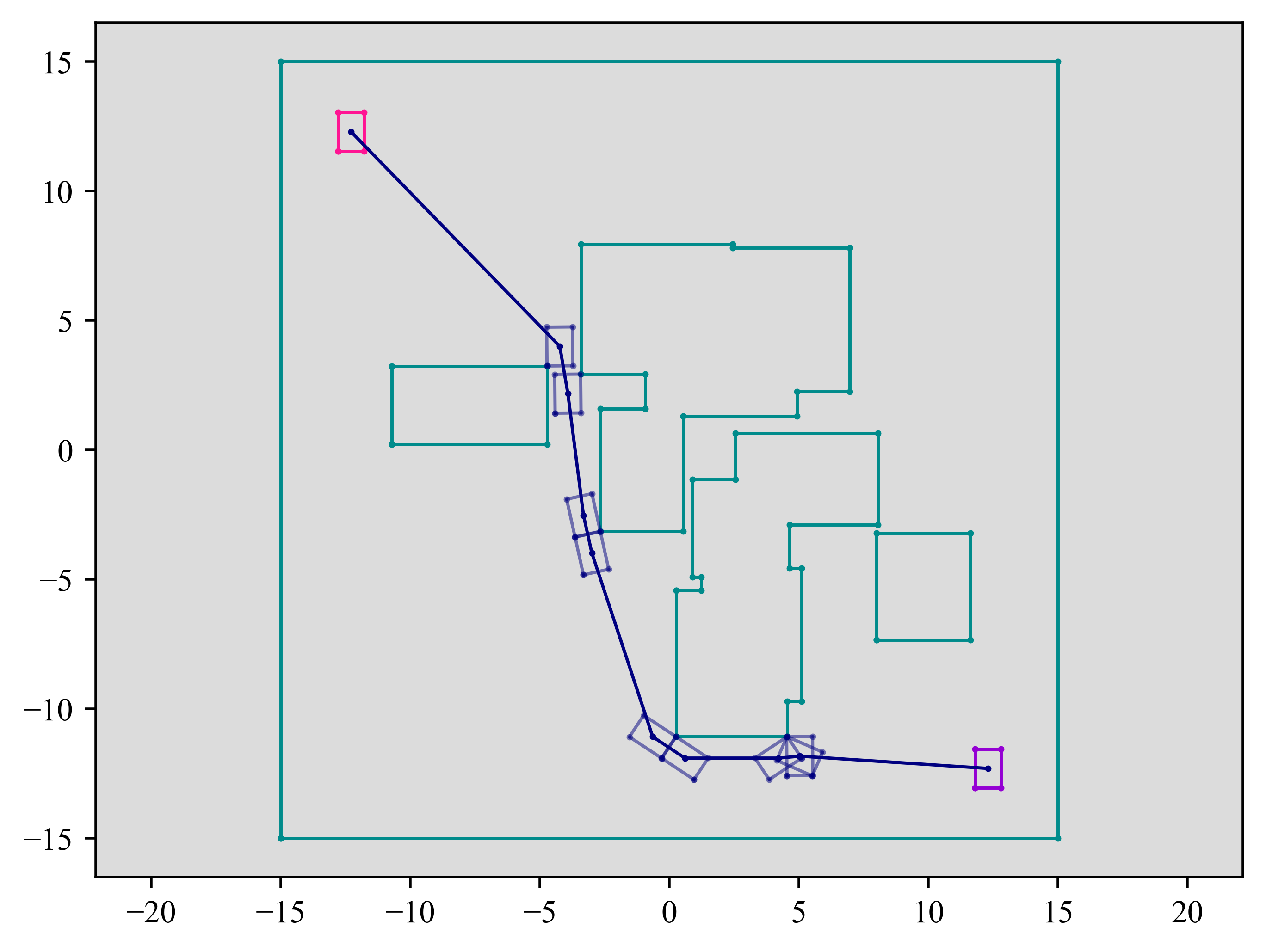}\hfill
\includegraphics[width=0.095\linewidth, trim=3.05cm 1.23cm 2.25cm 0.66cm, clip]{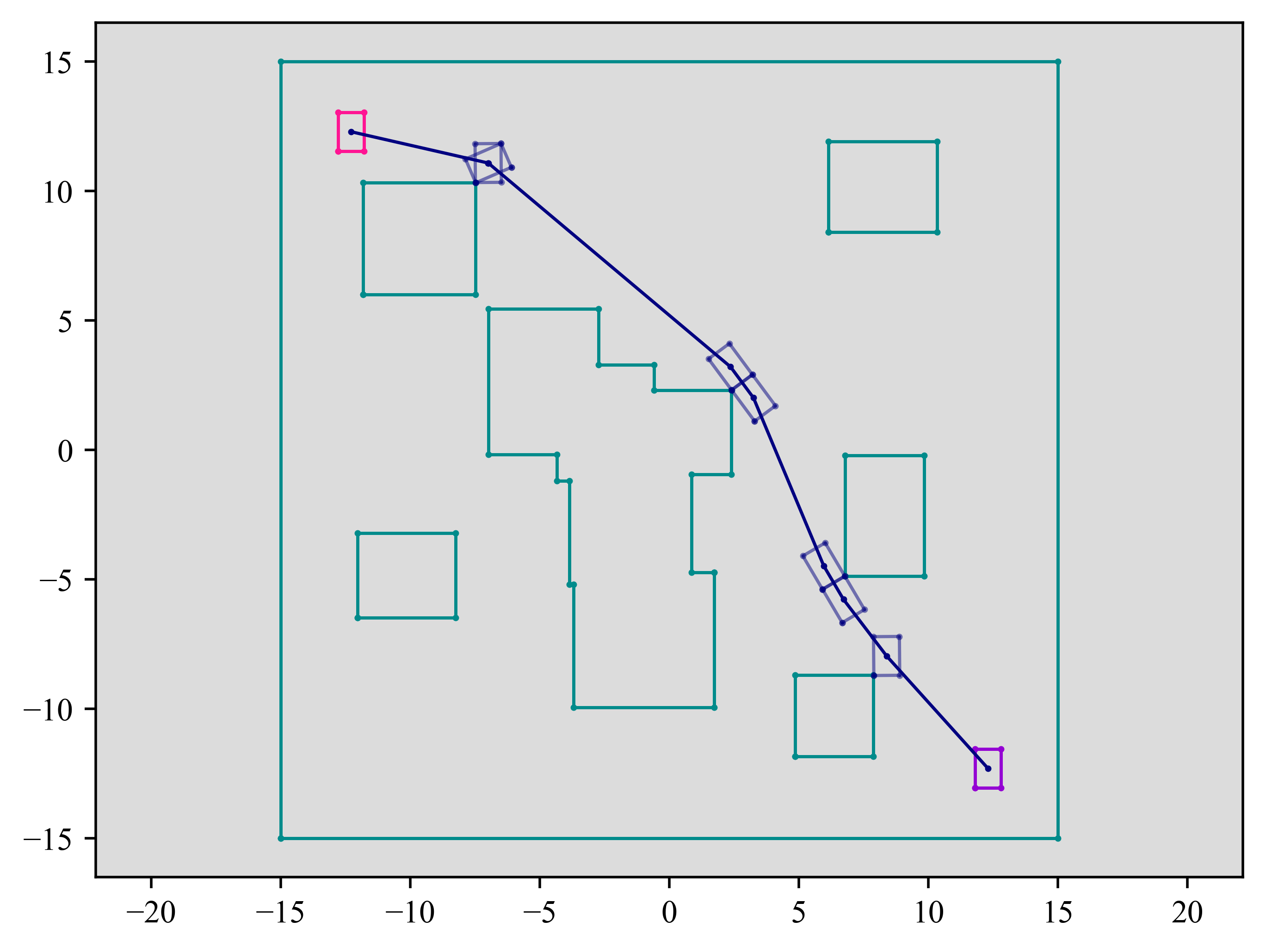}\hfill
\includegraphics[width=0.095\linewidth, trim=3.05cm 1.23cm 2.25cm 0.66cm, clip]{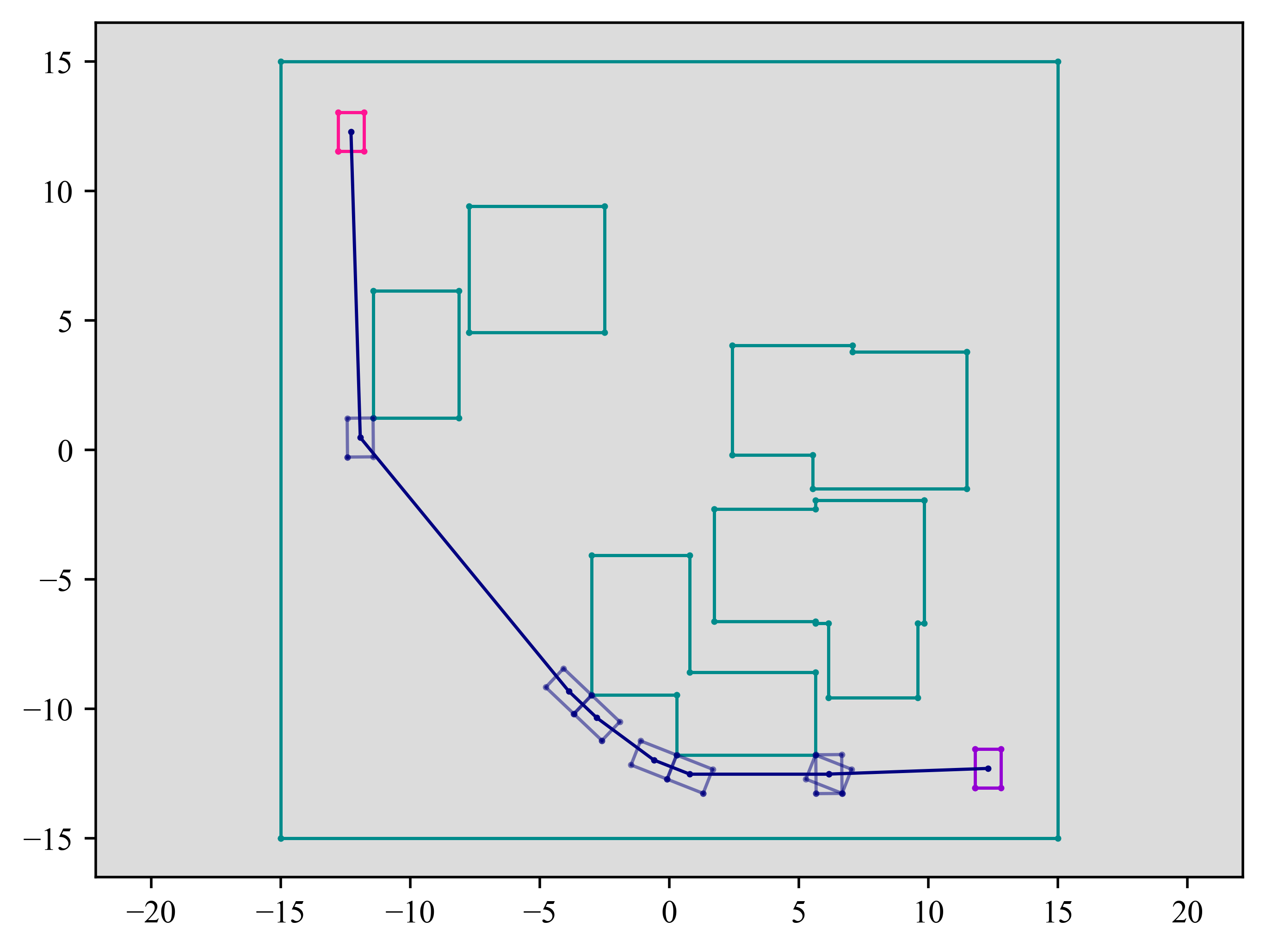}\hfill
\includegraphics[width=0.095\linewidth, trim=3.05cm 1.23cm 2.25cm 0.66cm, clip]{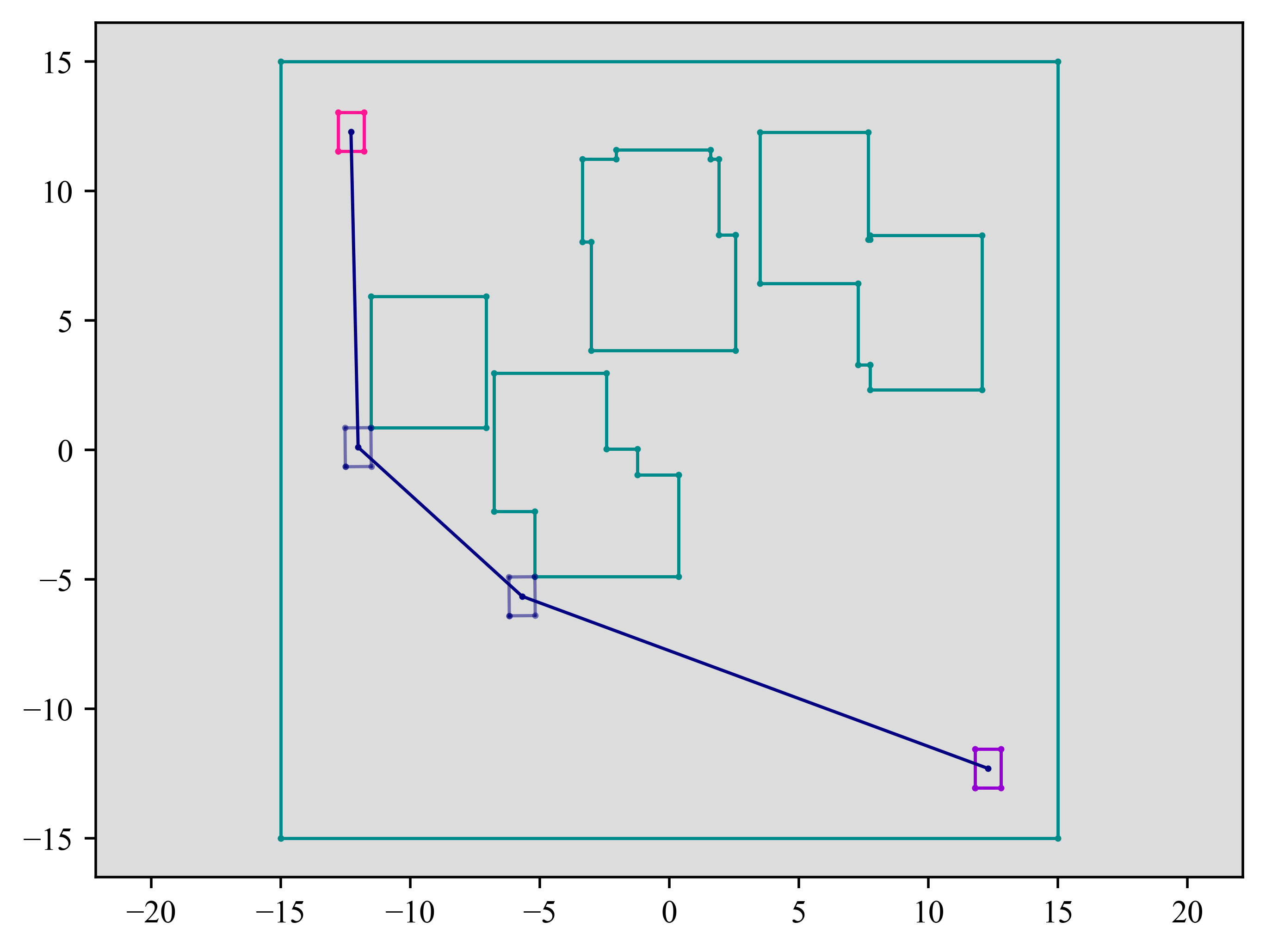}\hfill
\includegraphics[width=0.095\linewidth, trim=3.05cm 1.23cm 2.25cm 0.66cm, clip]{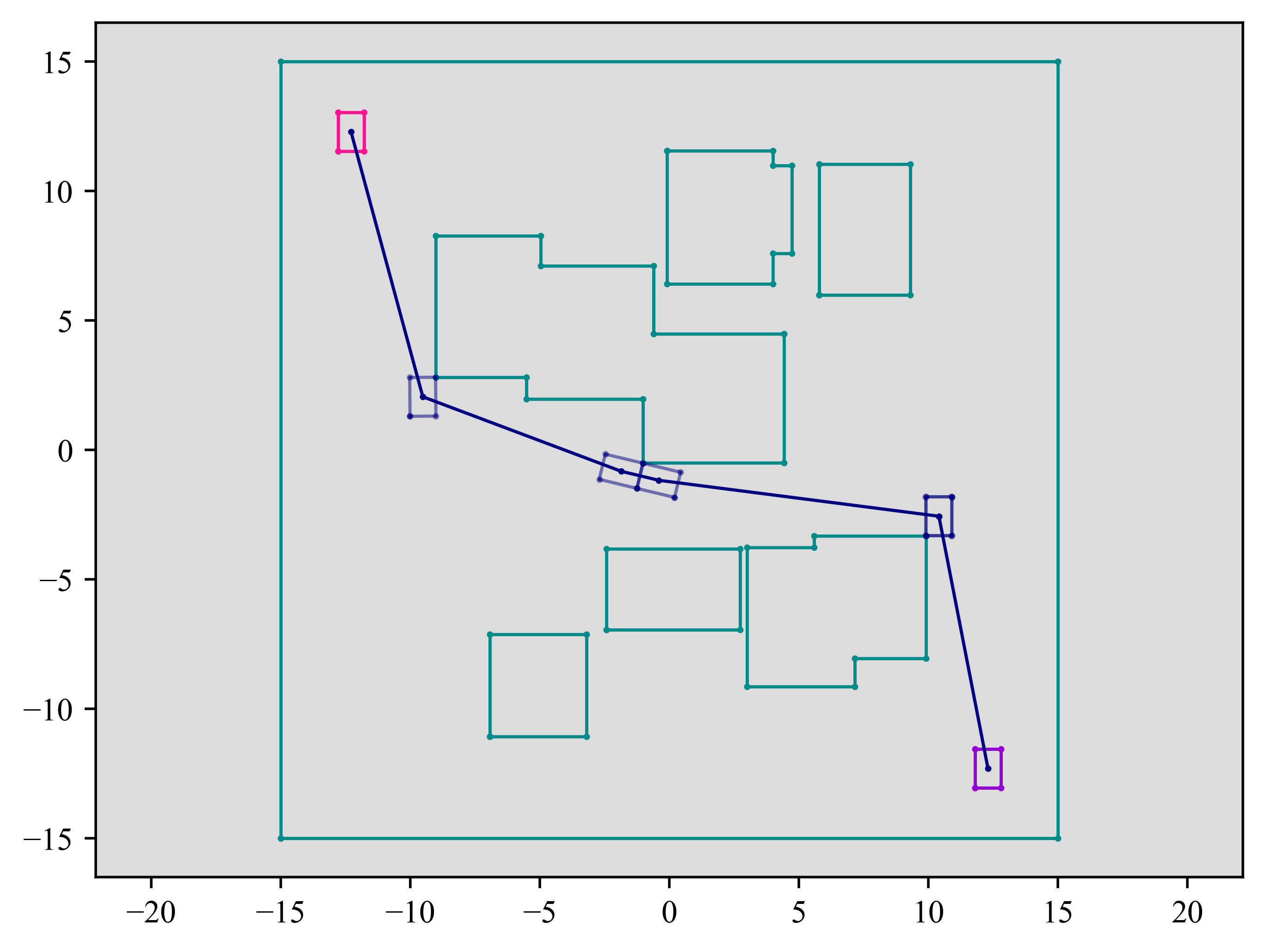}
\caption{The $10$ randomly generated simple maps used in our evaluation, where we ask a rectangular robot to travel from the top left of the map to the bottom right of the map. For each map, the (Euclidean distance-based) shortest path, found at a resolution of $360$, is also illustrated.} 
\label{fig:maps-simple}
\end{figure*}

\subsection{Quantitative Performance of \ours}
For a fair comparison, we generate a set of random benchmarking problems similar to these used in \cite{karaman2011sampling,strub2020adaptively} for evaluating the performance of \ours, using the following procedure: Within predefined map boundaries, $(x, y)$ positions are uniformly randomly sampled for placing convex polygonal obstacles. Each polygonal obstacle is also randomly created by sampling edges according to some proper rules. Intersecting polygons are merged. 

Two example problems are illustrated in Fig.~\ref{fig:eval-env}. The left subfigure shows the ``simple'' setting with a relatively small number of obstacles (total number of polygons sampled: 15). The right subfigure shows the ``hard'' setting with a large number of obstacles (total number of polygons sampled: 100).
For both, the projected $\RVG$s at a resolution of $36$ are shown, together with paths with the shortest distances as computed by \ours. 

\begin{figure}[h!]
\centering
\includegraphics[width=0.4\linewidth, trim=21.3cm 8.45cm 15.55cm 4.5cm, clip]{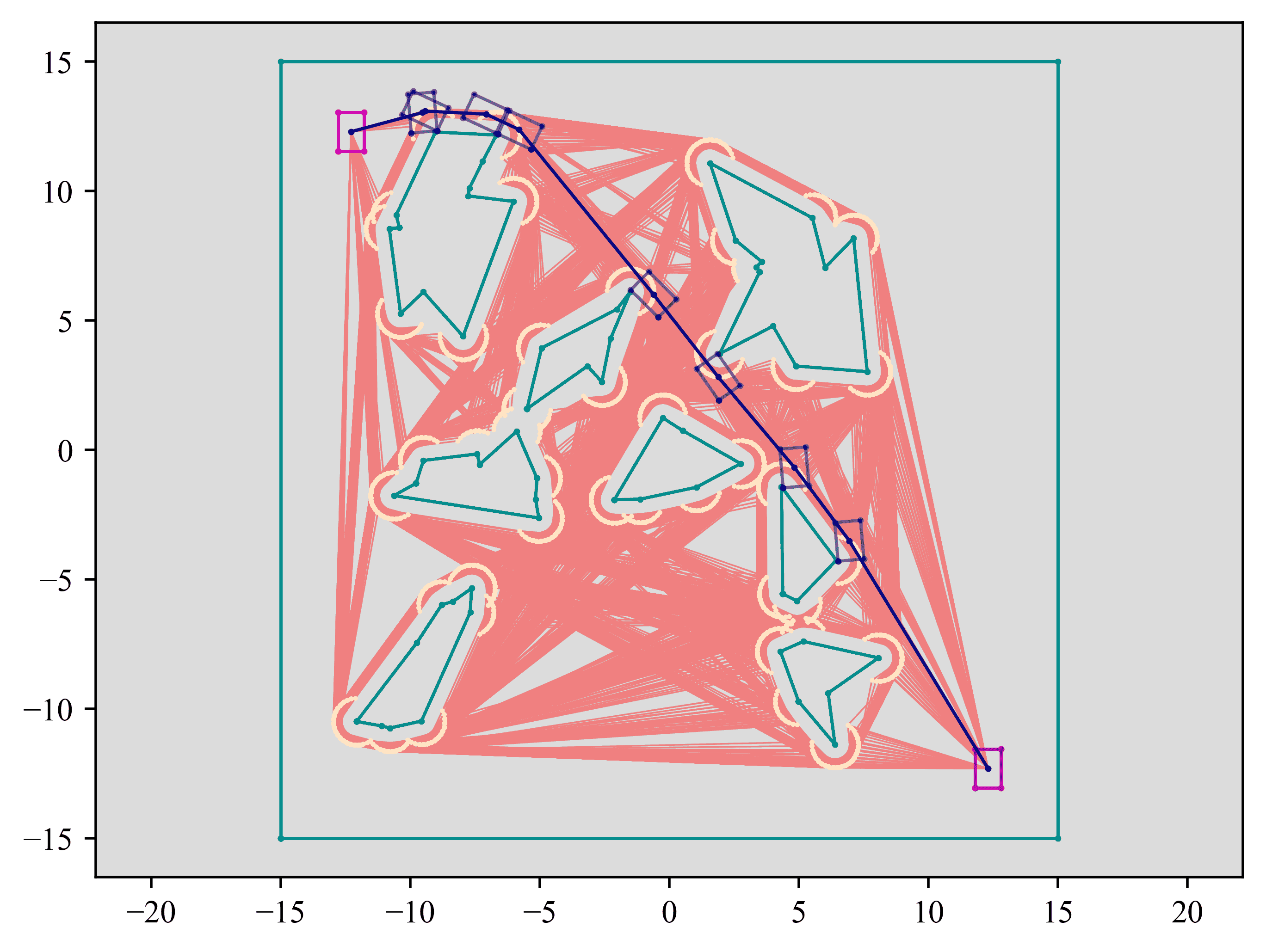}\hspace{7mm}
\includegraphics[width=0.4\linewidth, trim=21.3cm 8.45cm 15.55cm 4.5cm, clip]{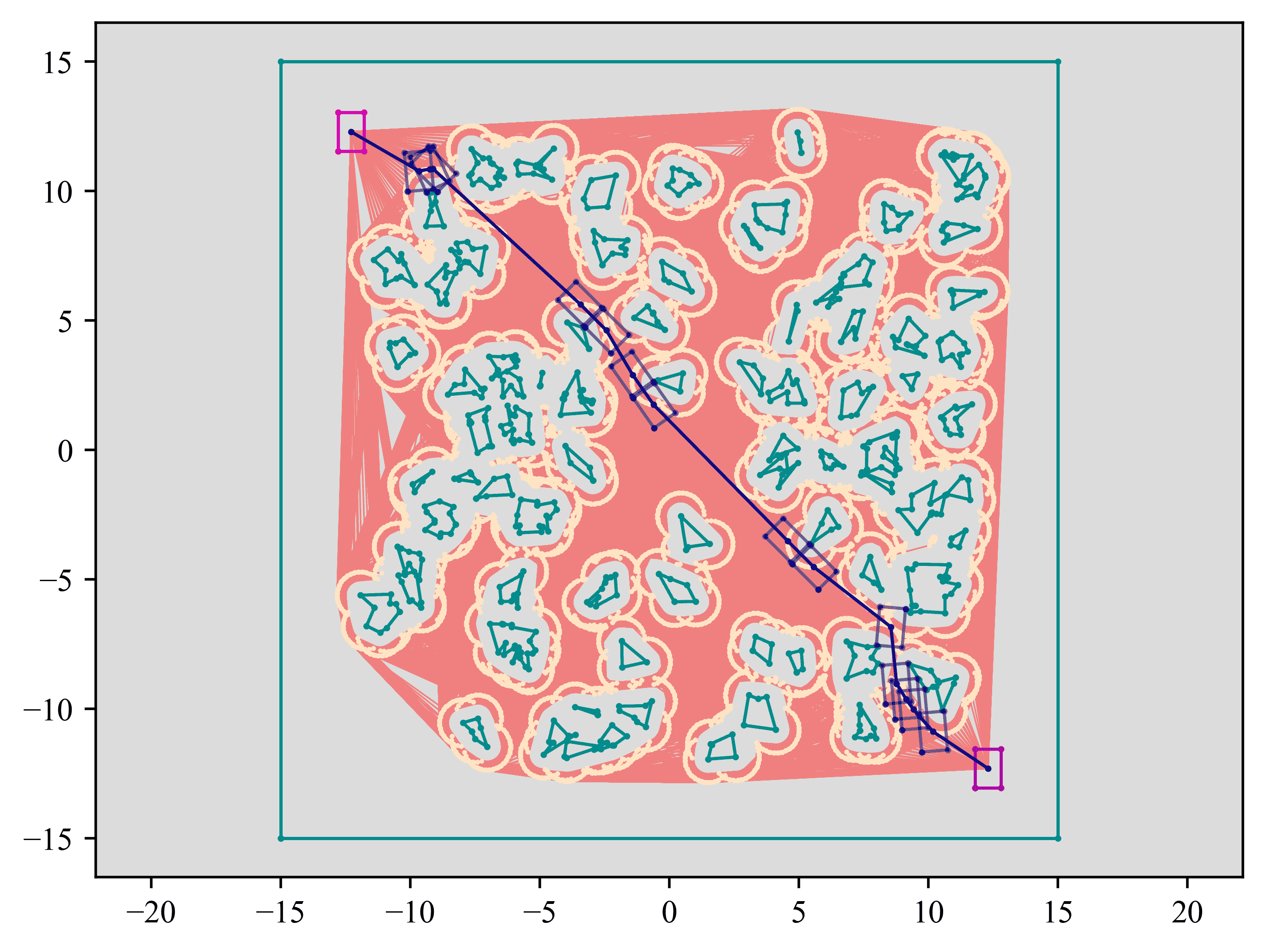}
\caption{Two randomly generated environments used in our performance evaluation. The maps, the $\RVG$s, the start and goal configurations, and the shortest paths found on the $\RVG$ are shown.   [Left] A \emph{simple} case. [Right] A \emph{hard} case.}
\label{fig:eval-env}
\end{figure}

\subsubsection{Performance of \ours}
We begin by benchmarking the standalone performance of \ours on computational speed and solution quality. First, for simple maps, for reference, the $10$  problems used for evaluation are shown in Fig.~\ref{fig:maps-simple} together with the highest quality (Euclidean) shortest path found by \ours at a resolution of $360$. We observe that the problems are rather diverse and already somewhat challenging (causing problems for sampling-based methods). In Fig.~\ref{fig:benchmark-simple}, the computational time for map building and search, as well as the path cost (length in this case since $\alpha = 1$) are shown for each map, for resolutions from $8$ to $360$. 

\begin{figure}[h!]
\centering
\includegraphics[width=\linewidth]{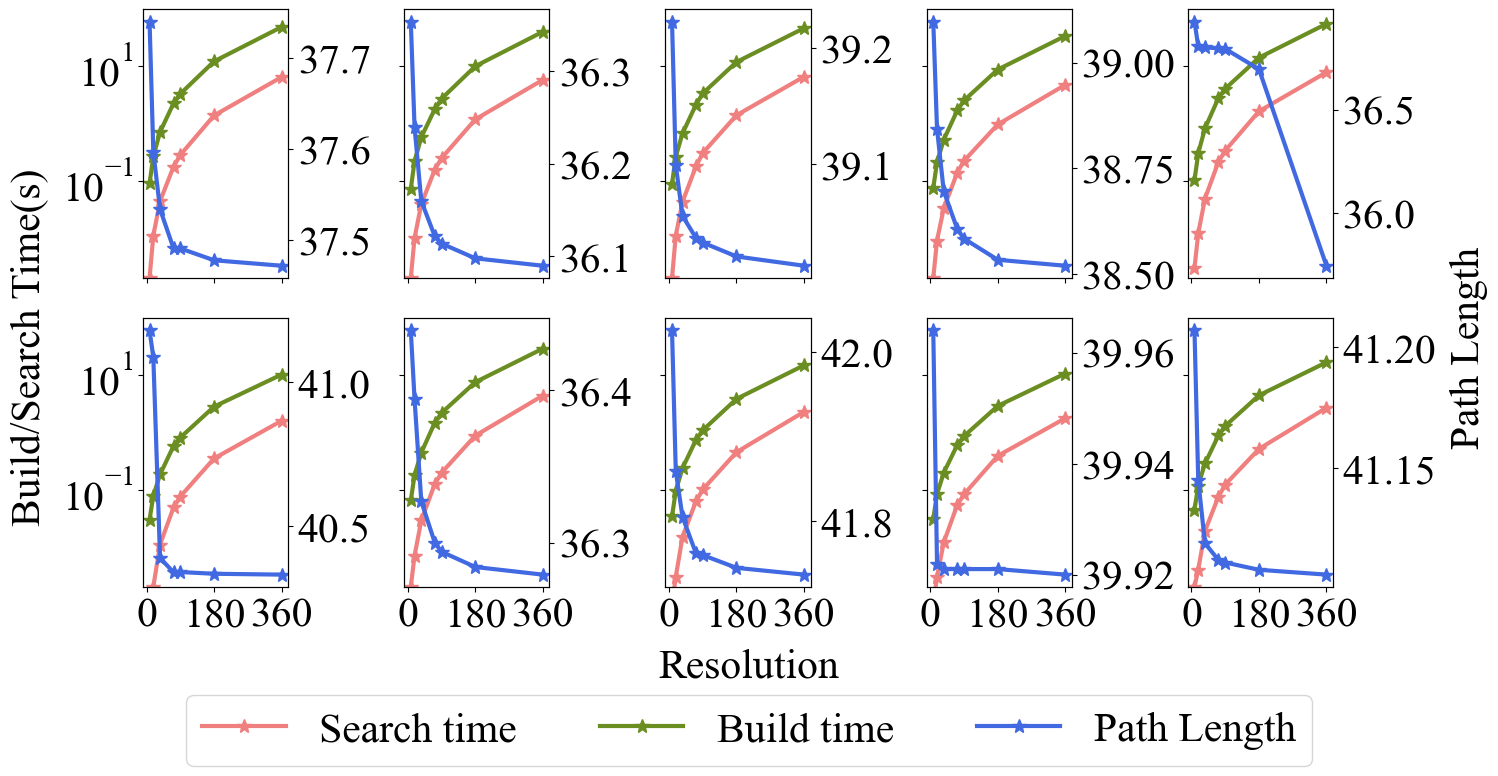}
\caption{Benchmark results for the $10$ randomly generated maps given in Fig.~\ref{fig:maps-simple}. The $\RVG$ building time, path search time, and path length ($\alpha = 1$, $\beta =0$ in Eq.~\eqref{eq:cost}) are shown for resolutions $8, 18, 36, 72, 90, 180$, and $360$.}
\label{fig:benchmark-simple}
\end{figure}

We observe that building the $\RVG$ can be done well under a second at low resolutions, a few seconds at medium resolutions, and tens of seconds at high resolutions. Searching for the optimal solution on the $\RVG$ generally takes a fraction of a second except at the highest resolutions. For the majority of the randomly sampled evaluation test cases, most near-optimal paths (through manual examination of the map, and shown in Fig.~\ref{fig:maps-simple}) are found at low resolutions, generally taking a total of no more than $10$ seconds. 

In Fig.~\ref{fig:benchmark-hard}, benchmarking results similar to that in Fig.~\ref{fig:benchmark-simple} are given for $10$ randomly generated hard maps, one of which is given in the right subfigure of Fig.~\ref{fig:eval-env}. The general trends here mimic those observed in Fig.~\ref{fig:benchmark-simple}, though the computation times now take longer. Again, we can compute a high-quality solution for most maps at low to medium resolution, taking tens of seconds. 

\begin{figure}[h!]
\centering
\includegraphics[width=\linewidth]{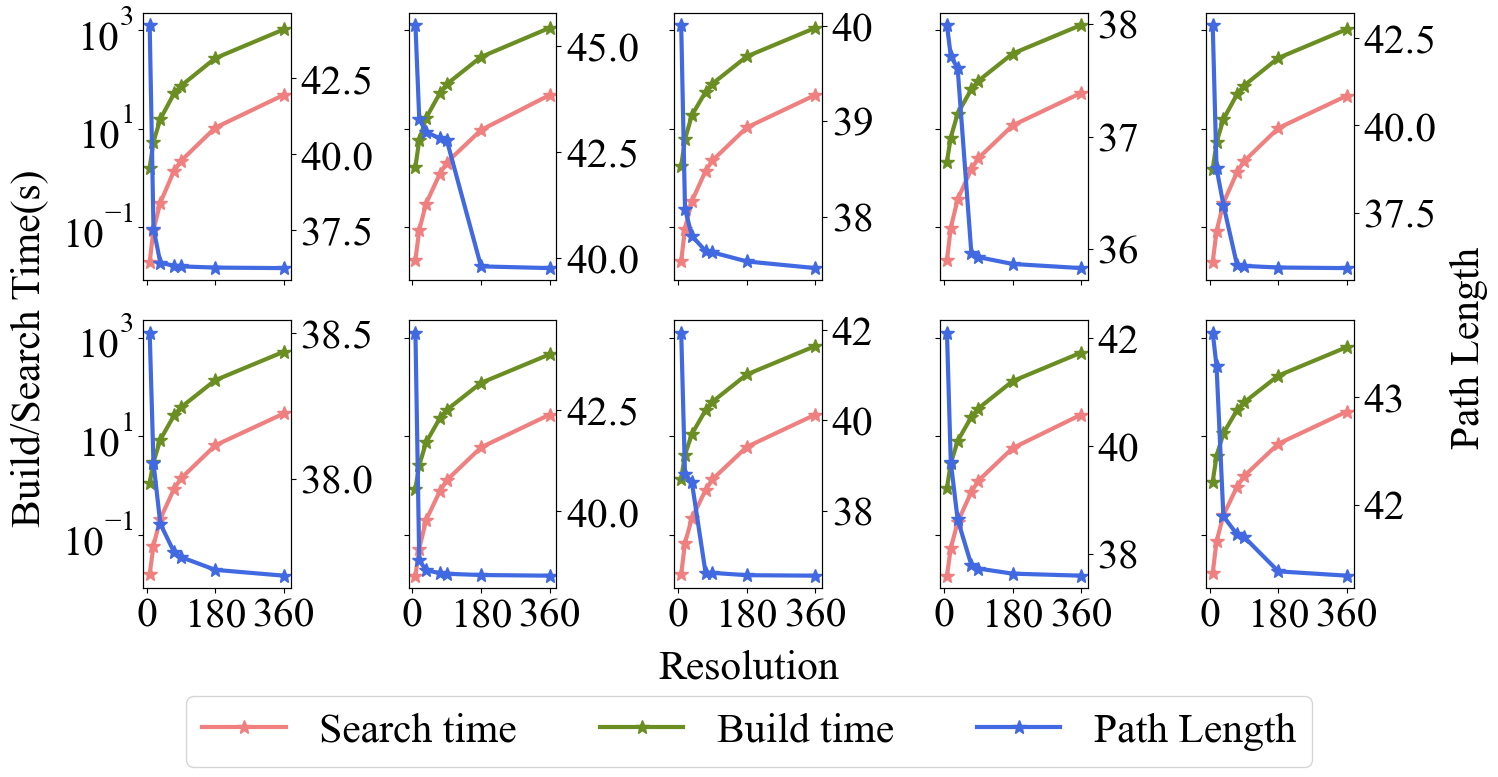}
\caption{Benchmark results for $10$ randomly generated hard problems. The $\RVG$ building time, path search time, and path length ($\alpha = 1$, $\beta =0$ in Eq.~\eqref{eq:cost}) are shown for resolutions $8, 18, 36, 72, 90, 180$, and $360$.}
\label{fig:benchmark-hard}
\end{figure}

\begin{figure*}[t!]
\centering
\includegraphics[width=\linewidth]{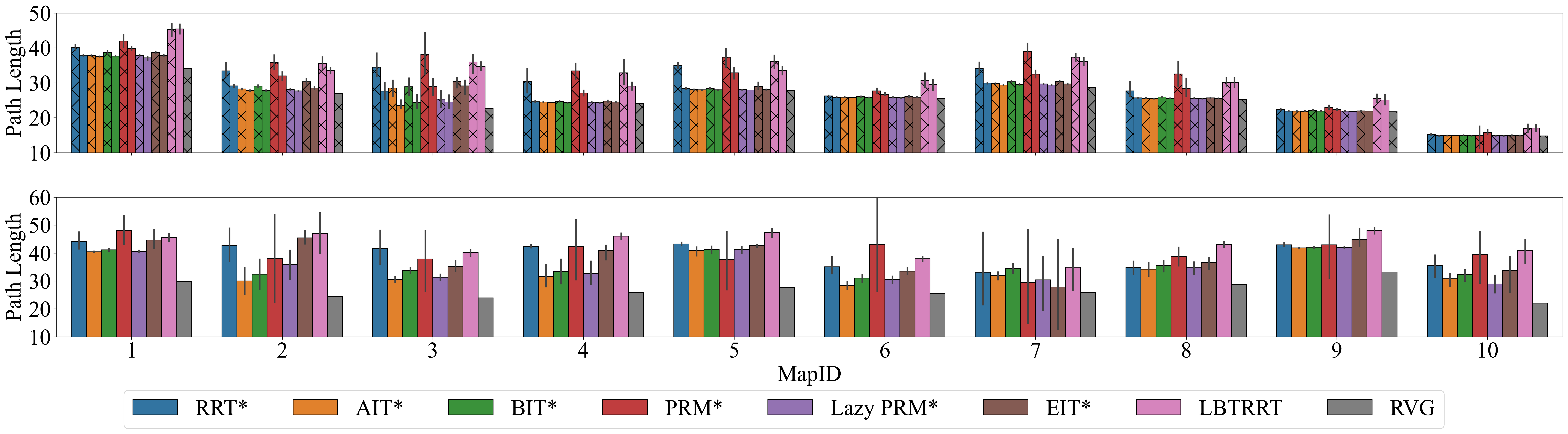}
\caption{[Top] Average path lengths over $10$ runs for $10$ environments returned by different motion planners with 1s/10s planning time vs \ours with 0.5s planning time on average (patterned/plain bars represent the results with 1s/10s running time). [Bottom] Average path length over 10 runs for 10 hard environments returned by different motion planners with 20s planning time versus \ours with 12.6s planning time on average.}
\vspace{1mm}
\label{fig:compare}
\end{figure*}
Fig.~\ref{fig:benchmarking-relative} presents a different view of the results provided in Fig.~\ref{fig:benchmark-simple} and Fig.~\ref{fig:benchmark-hard}. Here, computation times and solution path lengths are scaled based on the times/lengths at the highest resolution $360$, as the resolution varies. We note that a resolution of $360$ is difficult to realize in practice, requiring the robot to move with no more than one degree of error. Therefore, solution costs obtained at this resolution can be regarded as the optimal solution in practice. Fig.~\ref{fig:benchmarking-relative} provides a clearer picture of when we hit the sweet spot in terms of good solution quality and fast computation. To reach a decent solution quality, e.g., around $1.02$-optimal as compared with the solution quality at the highest resolution, a resolution of around $36$ is generally sufficient. 

\begin{figure}[h!]
\centering
\includegraphics[width=\linewidth]{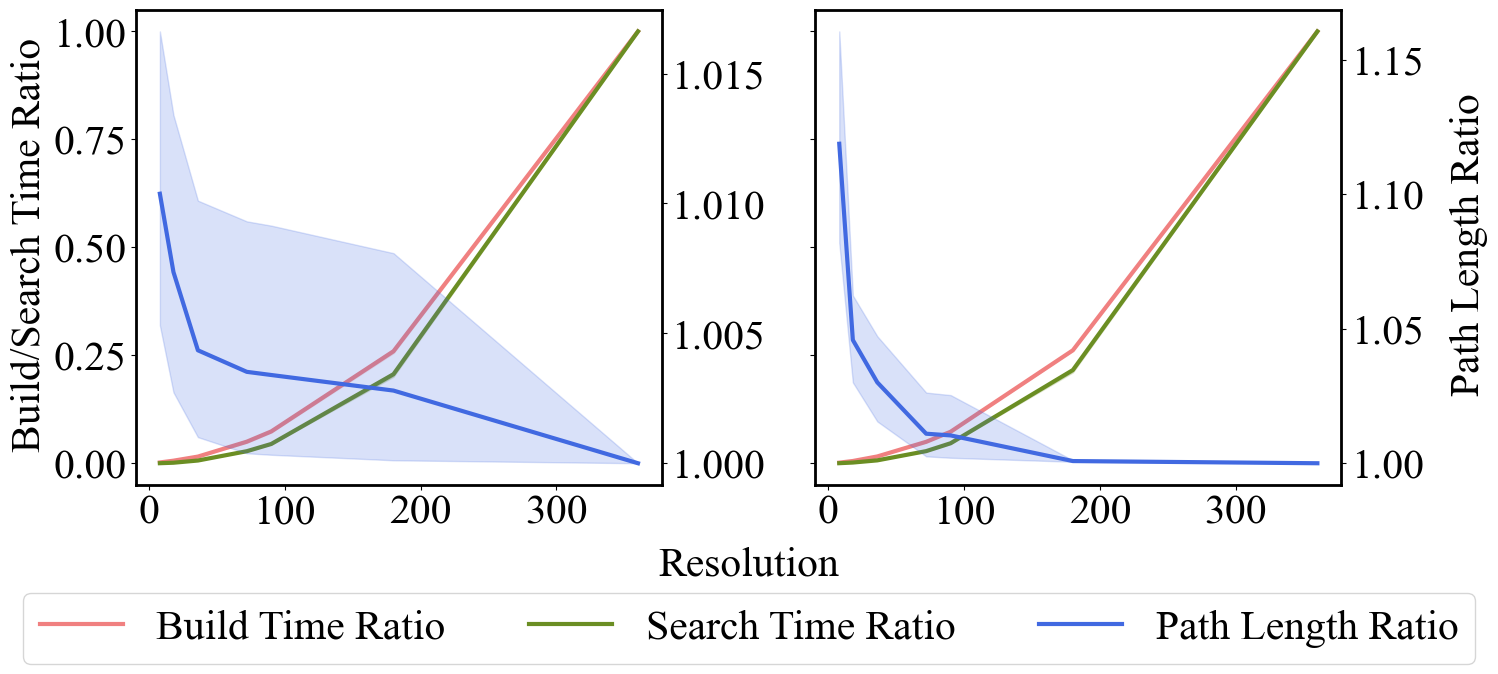}
\caption{$\RVG$ building time, path search time, and path length, scaled against these at the highest resolution. [Left] The ratios for the simple maps. [Right] The ratios for the hard maps.}
\label{fig:benchmarking-relative}
\end{figure}

\subsubsection{Comparison with State of the Art Sampling-Based Algorithms}
\ours, as a specialized multi-query method, shines even when compared with general single-query sampling-based planning methods. 
We use OMPL's implementations for all the other methods, and the collision checking is implemented with CGAL.
Our first comparison (the top subfigure of Fig.~\ref{fig:compare}) looks at the solution quality achieved for RRT*, BIT*, AIT*, EIT*, LBT-RRT, PRM*, LazyPRM*, and \ours where other methods use $1$ or $10$ seconds of planning time while \ours (at resolution $18$) uses 0.94 seconds on average for these maps to build the $\RVG$ and to execute the search. In all cases, \ours outperforms the state-of-the-art single-query sampling-based methods even when these sampling-based methods are allocated ten times the planning time budget. If multiple queries are performed, then \ours's advantage will become more obvious. 
A similar comparison experiment is carried out for the hard maps with the resolution at $18$ (the bottom subfigure of Fig.~\ref{fig:compare}). Under the hard settings, we observe that the optimality gaps between \ours and the compared single-query methods widen. 
In Fig.~\ref{fig:compare-ratio}, the following experiments are done on the simple maps. For each resolution of $8, 18, 36, 72, 90, 180$, and $360$, we solve the instances with \ours and record the total map building and searching time. Then, the sampling-based algorithms are given the same amount of time to work on the same map. We then compare the resulting solution optimality using the path length computed by \ours as the baseline. We observe that \ours again consistently outperforms by a significant margin, even under the single-query setting.
An analog for the hard maps is not provided, as performing the needed computation for the sampling-based methods will take hundreds of hours. 

\begin{figure}[h!]
\centering
\includegraphics[width=\linewidth]{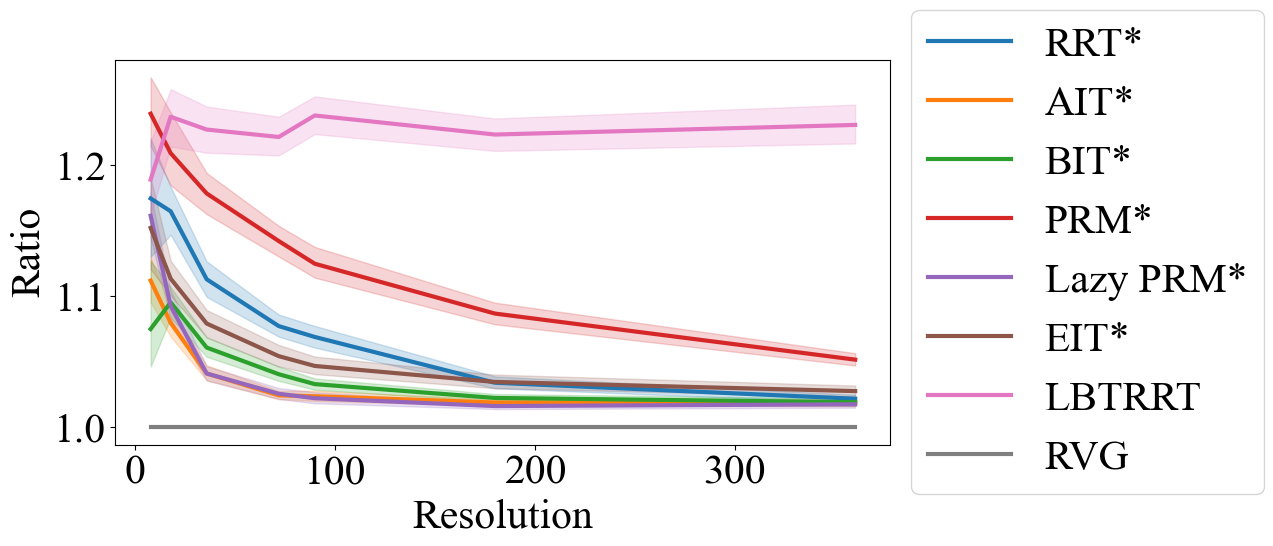}
\caption{Solution optimality ratio ($y$-axis) between sampling-based methods and \ours at different resolutions ($x$-axis) for simple maps, where each sampling-based method is executed for $10$ times per data point with each run using the time budget used by \ours (including map building and path search) at the given resolution.}
\label{fig:compare-ratio}
\vspace{1mm}
\end{figure}

\section{Conclusion and Discussions}\label{sec:conclusion}
In this study, we build the \ours algorithm to explore the natural idea of slicing the rotational degree of freedom, computing a 2D visibility graph for each slice, and connecting them to form a \emph{rotation-stacked visibility graph} ($\RVG$), which can be subsequently used to carry out multiple approximate shortest-path queries for a polygonal holonomic robot traversing a 2D polygonal environment. After a $\RVG$ building phase, \ours supports path queries of different cost metrics as specified in Eq.~\eqref{eq:cost}. We establish that \ours is resolution complete and asymptotically optimal, and show that it achieves highly favorable results in comparison to state-of-the-art sampling-based methods.

Our study opens many follow-up questions; we mention a few here. First, only holonomic robots are currently supported. It would be interesting to explore extensions to $\RVG$ to support other types of robots, e.g., car-like robots \cite{reeds1990optimal}. 
Second, while \ours achieves decent performance, there is much room to improve its computational performance further. An obvious starting point to employ faster $\VG$ building and searching algorithms, e.g., adopting options mentioned in~\cite{mitchell2017shortest} and/or through parallelization of computation. As \ours matures further, it can potentially open up more applications, e.g., for autonomous driving. 

{\small
\bibliographystyle{formatting/IEEEtran}
\bibliography{bib/_main}
}


\end{document}